\documentclass{article}
\pdfoutput=1

\usepackage{microtype}
\usepackage{graphicx}
\usepackage{subfigure}
\usepackage{booktabs} 
\usepackage{amssymb}
\usepackage{bm}
\usepackage{bbm}
\usepackage{amsmath}  
\usepackage{lipsum}
\usepackage{amsthm}
\usepackage{xcolor}
\usepackage{textcomp}
\usepackage{multirow}
\usepackage{mathtools}
\usepackage{authblk}
\usepackage[numbers, sort]{natbib}
\usepackage{xspace}
\usepackage{comment}
\usepackage{thmtools, thm-restate}  
\usepackage{verbatim}
\usepackage{pifont}
\newcommand{\xmark}{\ding{55}}
\usepackage{adjustbox} 
\usepackage{algorithm}
\usepackage{algorithmic}

\usepackage{tikz}
\usetikzlibrary{arrows}
\usetikzlibrary{positioning}

\tikzset{
  treenode/.style = {align=center, inner sep=0pt, text centered,
    font=\sffamily},
  arn_n/.style = {treenode, circle, black, font=\sffamily\bfseries, draw=black,
    fill=white, text width=1.5em},
  arn_r/.style = {treenode, circle, black, font=\sffamily\bfseries, draw=black,
    fill=white, text width=1.0em},
  arn_x/.style = {treenode, rectangle, draw=black,
    minimum width=0.5em, minimum height=0.5em}
}

\definecolor{tabblue}{HTML}{4e79a7}
\definecolor{tabred}{HTML}{e15759}
\usepackage[hidelinks,colorlinks=true,linkcolor=tabred,citecolor=tabblue]{hyperref}

\newif\ifarxiv
\newif\ifmefomo
\arxivtrue

\ifmefomo
\usepackage{mefomo_2023, times}
\else
\ifarxiv
\usepackage[margin=1in]{geometry}
\usepackage{authblk}
\else
\usepackage{icml2022}
\icmltitlerunning{Simple Hardware-Efficient Long Convolutions for Sequence Modeling}
\fi

\usepackage{bm}
\usepackage{enumitem}

\newtheorem{proposition}{Proposition}

\ifdefined\ShowNotes
  \newcommand{\colornote}[3]{{\color{#1}\bf{#2 #3}\normalfont}}
\else
  \newcommand{\colornote}[3]{}
\fi

\definecolor{darkred}{rgb}{0.7,0.1,0.1}
\definecolor{darkgreen}{rgb}{0.1,0.5,0.1}
\definecolor{cyan}{rgb}{0.7,0.0,0.7}
\definecolor{dblue}{rgb}{0.2,0.2,0.8}
\definecolor{maroon}{rgb}{0.76,.13,.28}
\definecolor{burntorange}{rgb}{0.81,.33,0}
\definecolor{royalpurple}{rgb}{0.47,.31,0.66}

\ifdefined\ShowNotes
  \newcommand{\num}[1]{{\color{red}\bf{#1}\normalfont}}
\else
  \newcommand{\num}[1]{#1}
\fi

\newcommand{\vA}{\mathbf{A}}
\newcommand{\vB}{\mathbf{B}}
\newcommand{\vC}{\mathbf{C}}
\newcommand{\vD}{\mathbf{D}}

\newcommand{\vF}{\mathbf{F}}
\newcommand{\vI}{\mathbf{I}}

\newcommand{\vK}{\mathbf{K}}
\newcommand{\vV}{\mathbf{V}}

\newcommand{\vP}{\mathbf{P}}

\newcommand{\bR}{\mathbb{R}}

\newcommand{\fastconv}{\textsc{FlashButterfly}\xspace}
\newcommand{\squash}{\textsc{Squash}\xspace}
\newcommand{\smooth}{\textsc{Smooth}\xspace}

\ifarxiv
\else
\usepackage[compact]{titlesec}
\titlespacing{\section}{0pt}{*0.3}{*0}
\titlespacing{\subsection}{0pt}{*0.15}{*0}

\usepackage[subtle, mathdisplays=tight, charwidths=tight, leading=normal]{savetrees}

\fi

\begin{document}

\ifmefomo
\title{Simple Hardware-Efficient Long Convolutions for Sequence Modeling}
\else
\ifarxiv
\title{Simple Hardware-Efficient Long Convolutions for Sequence Modeling}
  \author[$\dagger$]{Daniel Y. Fu\thanks{Equal Contribution.}}
  \author[$\ddagger$]{Elliot L. Epstein$^*$}
  \author[$\S$]{Eric Nguyen}
  \author[$\dagger\dagger$]{Armin W. Thomas}
  \author[$\dagger$]{Michael Zhang}
  \author[$\dagger$]{Tri~Dao}
  \author[$\ddagger\ddagger$]{Atri Rudra}
  \author[$\dagger$]{Christopher R{\'e}}
  \affil[$\dagger$]{Department of Computer Science, Stanford University}
  \affil[$\ddagger$]{Institute of Computational and Mathematical Engineering, Stanford University}
  \affil[$\S$]{Department of Bioengineering, Stanford University}
  \affil[$\dagger\dagger$]{Department of Psychology, Stanford University}
  \affil[$\ddagger\ddagger$]{Department of Computer Science and Engineering, University at Buffalo, SUNY\vspace{4pt}}
  \affil[ ]{{\texttt{danfu@cs.stanford.edu}, \texttt{epsteine@stanford.edu}, \texttt{\{etnguyen,~athms,~mzhang20,~trid\}@stanford.edu}, \texttt{atri@buffalo.edu},~\texttt{chrismre@cs.stanford.edu}}}

  \date{February 13, 2023}

\else
\twocolumn[
\icmltitle{Simple Hardware-Efficient Long Convolutions for Sequence Modeling}



\icmlsetsymbol{equal}{*}

\begin{icmlauthorlist}
\icmlauthor{Firstname1 Lastname1}{equal,yyy}
\icmlauthor{Firstname2 Lastname2}{equal,yyy,comp}
\icmlauthor{Firstname3 Lastname3}{comp}
\icmlauthor{Firstname4 Lastname4}{sch}
\icmlauthor{Firstname5 Lastname5}{yyy}
\icmlauthor{Firstname6 Lastname6}{sch,yyy,comp}
\icmlauthor{Firstname7 Lastname7}{comp}
\icmlauthor{Firstname8 Lastname8}{sch}
\icmlauthor{Firstname8 Lastname8}{yyy,comp}
\end{icmlauthorlist}

\icmlaffiliation{yyy}{Department of XXX, University of YYY, Location, Country}
\icmlaffiliation{comp}{Company Name, Location, Country}
\icmlaffiliation{sch}{School of ZZZ, Institute of WWW, Location, Country}

\icmlcorrespondingauthor{Firstname1 Lastname1}{first1.last1@xxx.edu}
\icmlcorrespondingauthor{Firstname2 Lastname2}{first2.last2@www.uk}

\icmlkeywords{Machine Learning, ICML}

\vskip 0.3in
]


\fi

\ifmefomo
\maketitle
\else
\ifarxiv
\maketitle
\fi


 \begin{abstract}
State space models (SSMs) have high performance on long sequence modeling but require sophisticated initialization techniques and specialized implementations for high quality and runtime performance.
We study whether a simple alternative can match SSMs in performance and efficiency: directly learning long convolutions over the sequence.
We find that a key requirement to achieving high performance is keeping the convolution kernels smooth.
We find that simple interventions---such as squashing the kernel weights---result in smooth kernels and recover SSM performance on a range of tasks including the long range arena, image classification, language modeling, and brain data modeling.
Next, we develop \fastconv, an IO-aware algorithm to improve the runtime performance of long convolutions.
\fastconv appeals to classic Butterfly decompositions of the convolution to reduce GPU memory IO and increase FLOP utilization.
\fastconv speeds up convolutions by \num{2.2$\times$}, 
and allows us to train on
Path256, a challenging task with sequence length \num{64K}, where we set state-of-the-art by \num{29.1} points while training \num{7.2$\times$} faster than prior work.
Lastly, we introduce an extension to \fastconv that learns the coefficients of the Butterfly decomposition, increasing expressivity without increasing runtime.
Using this extension, we outperform a Transformer on WikiText103 by \num{0.2} PPL with \num{30\%} fewer parameters.

\ifmefomo
State space models (SSMs) have high performance on long sequence modeling but require sophisticated initialization techniques and specialized implementations for high quality and runtime performance.
We study whether a simple alternative can match SSMs in performance and efficiency: directly learning long convolutions over the sequence.
We find that simply squashing the long convolutional kernel weights is enough to match SSMs in performance on a range of tasks including the long range arena (LRA) and language modeling.
Next, to improve runtime performance, we develop \fastconv, an IO-aware algorithm to compute long convolutions efficiently.
\fastconv appeals to classic Butterfly decompositions of the convolution to reduce GPU memory IO and increase FLOP utilization.
\fastconv speeds up the LRA benchmark by \num{7.0$\times$} over Transformers,
and allows us to train on
Path256, a challenging task with sequence length \num{64K}, where we set state-of-the-art by \num{29.1} points while training \num{7.2$\times$} faster than prior work.
\fi

 \end{abstract}


\section{Introduction\label{sec:intro}}
\ifmefomo
A fundamental question in understanding foundation models is whether their success depends on specific architectures like attention, or whether simpler alternatives can also suffice.
\else
\fi
Recently, a new class of sequence models based on state space models (SSMs)~\citep{gu2022efficiently,li2022makes, hasani2022liquid,gupta2022diagonal} has emerged as a powerful general-purpose sequence modeling framework.
\ifmefomo
SSMs scale nearly linearly in sequence length and have shown state-of-the-art performance on a range of sequence modeling tasks, from long range sequence modeling~\citep{smith2022simplified} to language modeling~\citep{dao2022hungry, ma2022mega}.
\else
SSMs scale nearly linearly in sequence length and have shown state-of-the-art performance on a range of sequence modeling tasks, from long range modeling~\citep{smith2022simplified} to language modeling~\citep{dao2022hungry, ma2022mega}, computer vision~\citep{islam2022long,nguyen2022s4nd}, and medical analysis~\citep{tang2022spatiotemporal}.
\fi

However, SSMs rely on sophisticated mathematical structures to train effectively in deep networks~\citep{gu2022efficiently}.
These structures generate a convolution kernel as long as the input sequence by repeatedly multiplying a hidden \textit{state} matrix.
This process may be unstable~\citep{goel2022s} and requires careful hand-crafted initializations~\citep{gu2022train}, leaving practitioners with a dizzying array of choices and hyperparameters.
\ifmefomo
In this paper, we study whether we can replace the SSMs with an even simpler approach -- parameterizing the long convolution kernel directly.
\else
This begs the question, \textit{why not parameterize the long convolution kernel directly?}
\fi

There are two challenges that long convolutions face for sequence modeling.
The first is quality: previous attempts at directly parameterizing the convolution kernel have underperformed SSMs~\citep{romero2021ckconv,li2022makes}.
The second is runtime performance: long convolutions can be computed in $O(N \log N)$ FLOPS in sequence length $N$ using the Fast Fourier transform (FFT), but systems constraints often make them slower than quadratic algorithms, such as attention.
\ifmefomo
In this paper, we show that a simple regularization technique and an IO-aware convolution algorithm can address these challenges.
\else
In this paper, we show that simple regularization techniques and an IO-aware convolution algorithm can address these challenges.
The simplicity of the long convolution formulation further allows for connections to block-sparse matrix multiplication that increase expressivity beyond convolutions or SSMs.
\fi

\begin{figure*}
    \centering
    \ifmefomo
    \includegraphics{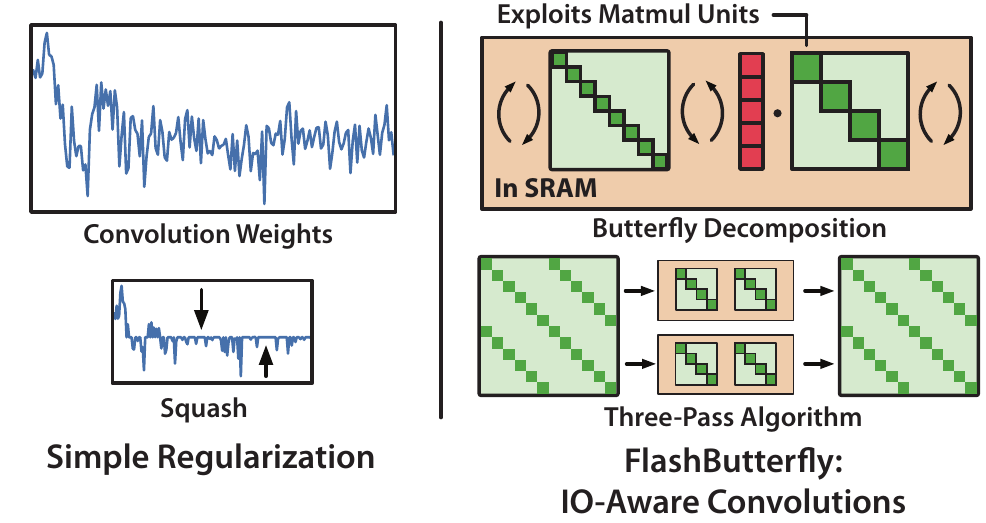}
    \else
    \ifarxiv
    \includegraphics[width=\textwidth]{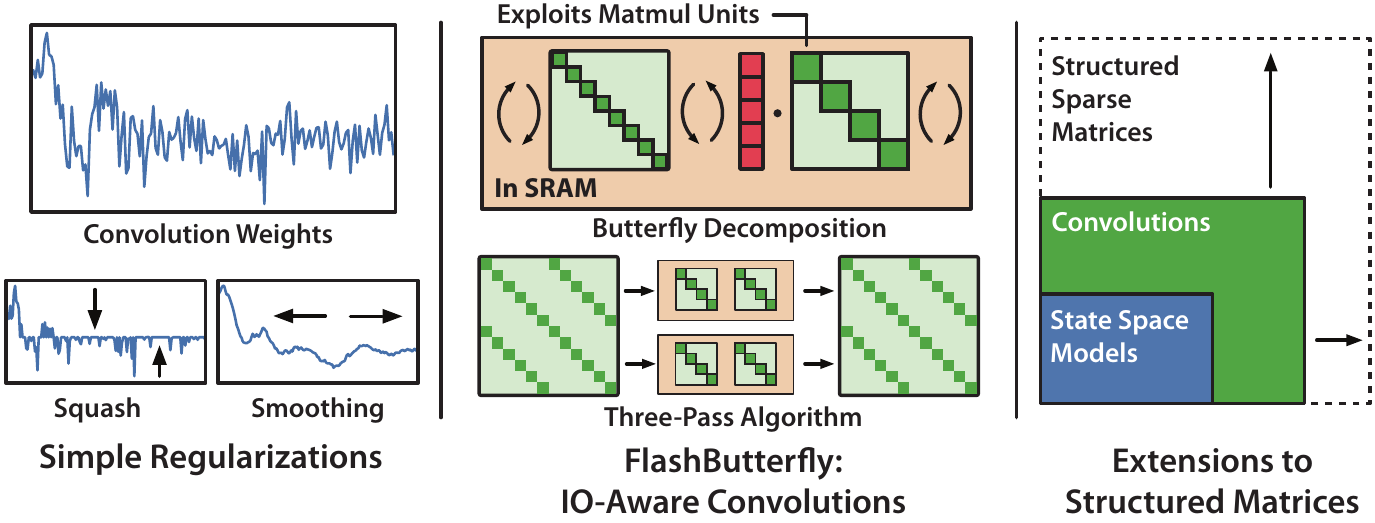}
    \else
    \includegraphics{figs/banner_pdf.pdf}
    \vspace{-1.5em}
    \fi
    \fi
    \caption{\label{fig:banner}
    \ifmefomo
    Left: A Simple regularization technique allow long convolutions to match state space models in sequence modeling.
    Right: We develop \fastconv, an IO-aware algorithm for long convolutions.
    \else
    Left: Simple regularization techniques allow long convolutions to match state space models in sequence modeling.
    Middle: \fastconv is an IO-aware algorithm for long convolutions that improves runtime performance and scales to long sequences.
    Right: We show deep connections to advances in block-sparse matrix multiplication and structured matrices.
    \fi
    }
    \ifmefomo
    \else
    \ifarxiv
    \else
    \vspace{-1.5em}
    \fi
    \fi
\end{figure*}
\ifarxiv
\paragraph{Closing the Quality Gap}
\else
\textbf{Closing the Quality Gap.}
\fi
First, to understand the quality gap, we study the performance of long convolutions compared to SSMs on Long Range Arena (LRA)~\citep{tay2020long}, a key benchmark designed to test long sequence models.
Long convolutions underperform SSMs by up to \num{16.6} points on average (Table~\ref{tab:lra}).
\ifmefomo
\else
Visualizing the convolution kernels identifies a potential culprit: the long convolution kernels are non-smooth, whereas SSM kernels are smooth (Figure~\ref{fig:kernels}).
\fi
\ifmefomo
\else

\fi
\ifmefomo
We find a simple regularization technique using a \squash operator to reduce the magnitude of the kernel weights.
Using this regularization, long convolutions also appear more robust to initialization than SSMs, matching hand-crafted S4 initializations on LRA even with \textit{completely random} initialization.
\else
We explore two simple regularization techniques from the signal processing literature that alleviate this problem. 
The first technique uses a \squash operator to reduce the magnitude kernel weights in the time domain, enforcing sparsity that translates to smoothness in the frequency domain.
The second technique applies a \smooth operator to the kernel weights in the time domain, which we find also promotes smoothness in the frequency domain.
With regularization, long convolutions recover the performance of SSMs---and appear more robust to initialization than SSMs, matching S4 on LRA even with \textit{completely random} initialization.
\fi

\ifmefomo
We further evaluate the performance of long convolutions on text modeling, where they are competitive with the recent H3 model~\citep{dao2022hungry}---coming within \num{0.3} PPL on OpenWebText~\citep{Gokaslan2019OpenWeb} and matching on the PILE~\citep{gao2020pile}.
\else
\fi

\ifmefomo
\else
Motivated by the success of these simple regularizations on LRA, we further evaluate the performance of long convolutions on other complex sequence modeling tasks from diverse modalities.
On image classification,
we find that long convolutions can be an effective drop-in replacement for SSM layers.
Replacing the SSM layer in S4 models with long convolutions yields a lift of \num{0.3} accuracy points on sequential CIFAR and comes within \num{0.8} points of S4ND-ISO on 2D CIFAR.
On text modeling, long convolutions are competitive with the recent SSM-based H3 model~\citep{dao2022hungry}---coming within \num{0.3} PPL of H3 on OpenWebText~\citep{Gokaslan2019OpenWeb} and matching H3 on the PILE~\citep{gao2020pile}.
Finally, long convolutions outperform both Transformers and SSMs in brain data modeling---by \num{0.14} and \num{0.16} MAE points, respectively---which suggests that the simpler architecture can even outperform SSMs for some applications.
\fi

\ifarxiv
\paragraph{Improving Runtime Performance}
\else
\textbf{Closing the Runtime Performance Gap.}
\fi
However, long convolutions are inefficient on modern hardware, since the FFT convolution incurs expensive GPU memory IO and cannot utilize matrix multiply units---even when using optimized implementations like cuFFT~\citep{cufft}.
SSM convolution formulations rely on specialized GPU Cauchy kernels and Vandermonde kernels, as well as special recurrent message passing structure, to overcome these challenges.

In response, we develop \fastconv, a simple IO-aware algorithm for long convolutions, which does not require ad hoc hand engineering.
\fastconv appeals to classic Butterfly decompositions of the FFT to rewrite the FFT convolution as a series of block-sparse Butterfly matrices.
This decomposition reduces the number of passes over the input sequence---reducing the GPU memory requirements---and utilizes matrix multiply units on the GPU, which increases FLOP utilization.

\ifmefomo
\else
\fastconv speeds up convolutions by \num{2.2$\times$} over cuFFT, and outperforms the fastest SSM implementations, since it does not incur the cost of generating the SSM convolution kernel.
\fi
To demonstrate \fastconv's scaling ability, we train a long convolution model on Path256, a task with sequence length 64K.
We set state-of-the-art by \num{29.1} points and train \num{7.2$\times$} faster than the previous best model.

\ifmefomo
\else
\ifarxiv
\paragraph{Deeper Connections and Learned Butterfly Extension}
\else
\textbf{Deeper Connections and Learned Butterfly Extension.}
\fi
The Butterfly decomposition in \fastconv forms deep connections to recent work in block-sparse matrix multiplication~\citep{chen2021pixelated}.
Butterfly matrices are a special case of Monarch matrices, which capture a large class of structured matrices~\citep{dao2022monarch}.
The block size $r$ interpolates between the fixed FFT for small block sizes to fully dense matrix multiplication for large matrices.
This connection suggests a natural \textit{learned Butterfly extension} that goes beyond convolutions in expressivity.

Our learned Butterfly extension simply learns the parameters in the Butterfly matrices from the data, instead of using the fixed matrices that corresopnd to the FFT and inverse FFT.
Learning the Butterfly matrices while keeping the block size fixed yields additional parameters without additional FLOPS---yielding \num{0.8} additional points of lift on sequential CIFAR.
Increasing the block size of the Butterfly matrices approaches the expressivity of fully dense matrices---including those used in linear layers and MLPs.
As a proof of concept, we use this property to replace the MLPs in a Transformer language model---and outperform a GPT-2 model on WikiText103 by \num{0.2} PPL with \num{30\%} fewer parameters.
\fi

\ifarxiv
\paragraph{Summary}
\else
\textbf{Summary.}
\fi
In summary, we show that long convolutions are an effective model for long sequence modeling.
They match or exceed SSMs across an array of diverse sequence domains while requiring less hand-crafted initializations and showing improved stability.
Additionally, by leveraging connections to Butterfly matrices, long convolutions can be trained up to \num{1.8$\times$} faster than SSMs.\footnote{Our code is available at \url{https://github.com/HazyResearch/safari}.}

\section{Background\label{sec:background}}

\paragraph{Deep State Space Models}
\ifmefomo
A discrete-time state space model (SSM) maps an input $u \in \mathbb{R}^N$, over time $t \in \{1,...,N\}$, to an output signal $y \in \mathbb{R}^N$
as $x_t = \mathbf{A}x_{t-1}+\mathbf{B}u_t$, $y_t = \mathbf{C}x_t+\mathbf{D}u_t$, 
by the use of hidden state $x_t \in \mathbb{R}^d$ and some set of matrices $\mathbf{A} \in \mathbb{R}^{d \times d}$, $\mathbf{D} \in \mathbb{R}^{1 \times 1}$,
$\mathbf{B} \in \mathbb{R}^{d \times 1}$, $\mathbf{C} \in \mathbb{R}^{1 \times d}$.
\else
A continuous-time state space model (SSM) maps an input signal $u(t) \in \mathbb{R}^N$, over time $t$, to an output signal $y(t) \in \mathbb{R}^N$ as
\begin{align*}
\dot{x}(t) =& \mathbf{A} x(t)+\mathbf{B} u(t) \\
y(t) =&  \mathbf{C} x(t) +\mathbf{D} u(t),
\end{align*}
by the use of hidden state $x(t) \in \mathbb{R}^d$ and some set of matrices $\mathbf{A} \in \mathbb{R}^{d \times d}$, $\mathbf{D} \in \mathbb{R}^{1 \times 1}$,
$\mathbf{B} \in \mathbb{R}^{d \times 1}$, $\mathbf{C} \in \mathbb{R}^{1 \times d}$.
Discretizing the SSM yields a recursion $x_t = \mathbf{A}x_{t-1}+\mathbf{B}u_t$, $y_t = \mathbf{C}x_t+\mathbf{D}u_t$.
\fi
By unrolling the recursion, $y$ can be written as a convolution between $u$ and 
a kernel $\vK$ that depends on $\mathbf{A}$, $\mathbf{B}$, $\mathbf{C}$:
\ifmefomo
$y = \vK \ast u + \vD u$.
\else
\begin{equation}
    \label{eq:conv_layer}
    y = \vK \ast u + \vD u. 
\end{equation}
\fi
\ifmefomo
\else
A key ingredient to training deep SSM models is proper initialization of the learnable  
matrices $\vA$, $\vB$, $\vC$, and $\vD$.
Initialization strategies often draw upon the HiPPO theory~\citep{gu2020hippo} on orthogonal polynomials, and involve the selection of measures and discretization strategies.
The parameters may also be unstable to learn, which can require custom learning rate schedules~\citep{gu2022train}.
\fi

\paragraph{FFT Convolution}
\ifmefomo
A standard approach to compute convolutions in $O(N \log N)$ in sequence length $N$ is to use the FFT convolution theorem.
\else
Computing the convolution in Equation~\ref{eq:conv_layer} can be costly for long sequences.
A standard approach is to compute the convolution using the FFT convolution theorem.
\fi
Then, the convolution can be computed as:
\ifmefomo
$y = u \ast \vK = \vF_N^{-1} \vD_{\vK} \vF_N u$,
\else
\begin{equation}
    \label{eq:conv_fft}
y = u \ast \vK = \vF_N^{-1} \vD_{\vK} \vF_N u,
\end{equation}
\fi
where $\vF_N$ denotes the DFT matrix of size $N$, and $\vD_{\vK} = \text{diag}(\vF_N \vK)$.
\ifmefomo
\else
This so-called FFT convolution scales in $O(N \log N)$ in sequence length $N$, but is often unoptimized on modern hardware (most optimized convolution operators focus on short convolutions, e.g., 3$\times$3).
\fi

\ifmefomo
\else
\paragraph{Runtime Performance Characteristics}
We provide a brief discussion of relevant factors affecting runtime performance.
Depending on the balance of computation and memory accesses, operations can be
classified as either compute-bound or memory-bound.
In compute-bound operations, the time accessing GPU memory is relatively small compared to the time spent doing arithmetic operations.
Typical examples are matrix multiply with large inner
dimension, and short convolution kernels with a large number of channels.
In memory-bound operations, the time taken by the operation is determined by the
number of memory accesses, while time spent in computation is much smaller.
Examples include most other operations:
elementwise (e.g., activation, dropout) and reduction (e.g., sum,
softmax, batch norm, layer norm).
\fi

\paragraph{Our Approach}
\ifmefomo
Rather than parameterizing $\vK$ with carefully initialized SSM matrices, we seek to directly parameterize the convolution kernel $\vK$.
\else
Rather than parameterizing $\vK$ with carefully initialized SSM matrices, we seek to directly parameterize the convolution $\vK$ in Equation~\ref{eq:conv_layer}.
\fi
Our goal is to replace the SSM layer with a learned convolution kernel as a drop-in replacement, while keeping the stacking and multi-head structure of SSM models (which can be thought of as multiple convolutional filters).
We also aim to make the FFT convolution runtime-performant on modern hardware.
\section{Method\label{sec:method}}

\ifmefomo
\else
\begin{table}[t]
    \caption{Accuracy on the \textsc{ListOps} task in LRA.}
    \label{tab:random_micro}
    \small
    \centering
    \begin{tabular}{rc} \toprule
    {Model} & {Accuracy}  \\ \midrule
    S4-LegS & 59.6 \\ \midrule
    Long Convs & 53.4 \\
    Long Convs, +\smooth & 59.8 \\
    Long Convs, +\squash & \textbf{60.3} \\
    Long Convs, +\squash, +\smooth & 59.7 \\
    \bottomrule 
    \end{tabular}
\end{table}
\fi
\ifmefomo
In Section~\ref{sec:long_conv}, we conduct an initial investigation into long convolutions for sequence modeling, and develop a simple regularization strategy based on our findings.
\else
In Section~\ref{sec:long_conv}, we conduct an initial investigation into long convolutions for sequence modeling, and develop two simple regularization strategies based on our findings.
\fi
Then, in Section~\ref{sec:fastconv}, we present \fastconv, an IO-aware algorithm for speeding up convolutions modeled after block-sparse matrix multiplication.
\ifmefomo
\else
Finally, we present an extension of \fastconv that leverages the block-sparse connection for additional expressivity.
\fi
\subsection{Long Convolutions for Sequence Modeling}
\label{sec:long_conv}

First, we conduct a brief investigation into the performance of vanilla long convolutions on sequence modeling, and we find a gap in quality.
\ifmefomo
We then propose a simple regularization technique for closing this gap.
\else
We then propose two simple regularization techniques for closing this gap.
\fi
\ifmefomo
\else
\begin{figure}[ht]
    \centering
    \ifarxiv
    \includegraphics[width=\textwidth]{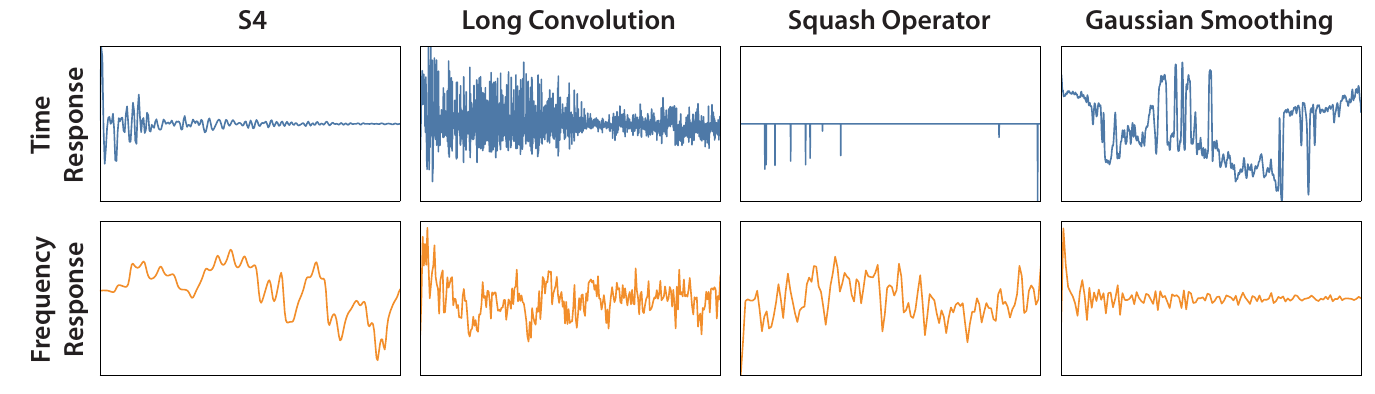}
    \else
    \includegraphics{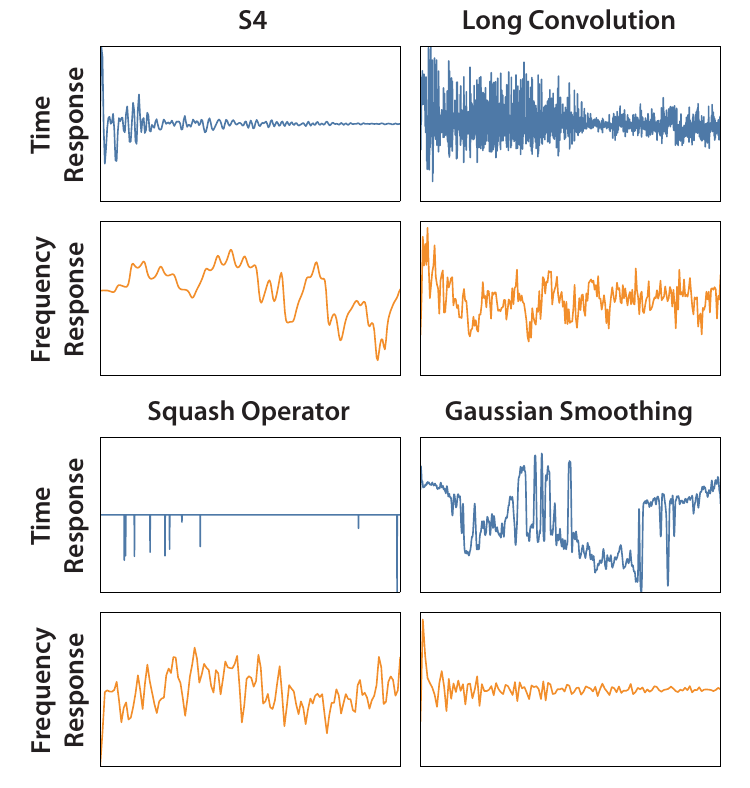}
    \fi
    \ifarxiv
    \else
    \vspace{-1em}
    \fi
    \caption{\label{fig:kernels}
    Visualizations of kernels trained on the \textsc{ListOps} task in LRA.
    Left to right, top to bottom: S4, long convolutions without regularization, long convolutions with the \squash operator, and long convolutions with the \smooth operator.
    Time response on top, frequency response on the bottom.
    }
    \ifarxiv
    \else
    \vspace{-1.5em}
    \fi
\end{figure} 
\fi
\ifmefomo
\else
\ifarxiv
\paragraph{Motivation for Regularization: Non-Smooth Kernels}
\else
\textbf{Motivation for Regularization: Non-Smooth Kernels.}
\fi
We begin by directly replacing the SSM layers in an S4 model with long convolutions, with random initialization.
We train a model on the \textsc{ListOps} task from the long range arena (LRA) benchmark~\citep{tay2020long}, with element-wise dropout on the convolution kernel weights.
Table~\ref{tab:random_micro} shows that long convolutions underperform SSMs with \num{6.2} points on \textsc{ListOps}.

To understand the gap in performance, we visualize one head of the convolution kernel $\vK$, compared to an SSM kernel in Figure~\ref{fig:kernels}.
Compared to well-initialized SSM kernels, we find that directly learning convolution weights results in convolution kernels that are non-smooth and appear noisy.
We hypothesize that these properties are responsible for the performance gap.
\fi

\ifmefomo
\else
\begin{table}[t]
    \caption{Convolution- and SSM-specific hyperparameters.}
    \ifarxiv
    \else
    \scriptsize
    \fi
    \label{tab:hyperparams}
    \centering
    \begin{tabular}{rll} \toprule
    {Model} & {Hyperparameters} & {Initializations} \\ \midrule
    SSM & $d$, $lr_A$, $lr_B$, $lr_C$& LegS, FouT, LegS/FouT \\
    & dropout, discretization & Inv, Lin \\
    Long Convs & $\lambda$, kernel LR, & Random, Geometric \\
    & k, dropout &  \\
    \bottomrule 
    \end{tabular}
\end{table}
\fi

\ifarxiv
\paragraph{Regularizing the Kernel}
\else
\textbf{Regularizing the Kernel.}
\fi
\ifmefomo
We begin by directly replacing the SSM layers in an S4 model with long convolutions, with random initialization. Table~\ref{tab:lra} shows that long convolutions underperform SSMs by \num{16.6} points on average across LRA.
We propose a simple technique for regularizing the convolution kernel using the \squash operator.
\else
We propose two simple techniques for regularizing the convolution kernel to alleviate these issues: \squash and \smooth.
\fi
The \squash operator is applied element-wise to the convolution kernel, and reduces the magnitude of all weights:
$\overline{\vK} = \text{sign}(\vK) \odot \text{max}(|\vK| - \lambda, 0)$.
\ifmefomo

\else
As an aside, we note that \squash is equivalent to taking one step of an L1 proximal operator:
$\overline{\vK} =\operatorname{Prox}_{\lambda \| .\|}(\vK) =\operatorname{argmin}_{x}\{\lambda \lVert x \rVert_1 + \lVert x-\vK \rVert_2^2 \}$
and thus may have principled connections to proximal gradient techniques.
\fi
\ifmefomo
\else
The \smooth operator applies simple average pooling, with width $p$, to the convolution kernel: $\overline{\vK}_k = (2p+1)^{-1}\sum_{j=1}^{2p+1} \vK_{k+j - p}.$
\ifarxiv

\else
\fi
Training long convolutions with these regularizations matches SSMs in performance on the \textsc{ListOps} task (Table~\ref{tab:random_micro}).
\ifarxiv
Additionally, Figure~\ref{fig:kernels} right shows that these regularizations improve smoothness in the frequency domain as well.
\else
Additionally, Figure~\ref{fig:kernels} bottom shows that these regularizations improve smoothness in the frequency domain as well.
\fi
\fi
In Appendix~\ref{sec:exp_supp}, we evaluate directly smoothing in frequency domain.

\ifarxiv
\paragraph{Initialization}
\else
\textbf{Initialization.}
\fi
\ifmefomo
\else
We seek to understand how sensitive long convolutions are to initialization.
We note that since $\vK$ directly parameterizes the convolution kernel, we can also leverage advances in initialization in SSMs such as HiPPO~\citep{gu2020hippo} and S4-LegS~\citep{gu2022train}---simply by converting the initialized SSM model to a convolution kernel, and initializing $\vK$ to the convolution weights.

While complex initialization strategies can be powerful, they require careful tuning to configure.
\fi 
To understand the impact of initialization on long convolutions, we evaluate two simple intialization techniques: random initialization, and a geometric decay initialization.
The random initialization initializes the weights to be randomly distributed from a Normal distribution: $\vK_i \sim \mathcal{N}$.
The geometric decay initialization additionally scales kernel weights to decay across the sequence, as well as across the heads.
For the kernel $\vK^{(h)}$, $1 \leq h \leq H$, we initialize the weights as:
$\vK^{(h)}_k = x \exp{(-kN^{-1}\left(H/2\right)^{hH^{-1}})},$
for $1 \leq k \leq N$, where $x \sim \mathcal{N}$ is drawn from a Normal distribution.

\ifmefomo
\else

\paragraph{Summary}
The full method is written in Algorithm~\ref{alg:AMMIT}, with a forward reference to our fast convolution solution \fastconv. 
In Algorithm~\ref{alg:AMMIT}, all operators ($\operatorname{max}$, $\operatorname{sign}$, and absolute value) are applied entry-wise, \fastconv is taken over the sequence dimension and the skip connection is taken over the head dimension.
\fi
\ifmefomo
\else
\ifarxiv
\begin{algorithm}[h]
\else
\begin{algorithm}[t]
\fi
    \algsetup{linenosize=\tiny}
    \caption{\label{alg:AMMIT} Regularized Long Convolution}
    \ifarxiv
    \else
    \small
    \fi
    \begin{algorithmic}[1]
      \REQUIRE Input $u \in \mathbb{R}^{B \times H \times N}$, $\vK \in \mathbb{R}^{H \times N}$, $\vD \in \mathbb{R}^{H}$, where $N$ is the sequence length, $H$ is the head dimension, and $B$ is the batch size.
      \STATE $\vK \gets \operatorname{dropout}(\vK)$.
      \ifmefomo
      \else
      \STATE $\vK_k \gets (2p+1)^{-1}\sum_{j=1}^{2p+1} \vK_{k+j - p}$.
      \fi
      \STATE ${\vK} \gets \text{sign}(\vK) \odot \text{max}(|\vK| - \lambda, 0)$.
      \STATE $y \gets \fastconv(\vK,u) + \vD \odot u.$
      \STATE Return $y \in \mathbb{R}^{B \times H \times N}$
    \end{algorithmic}
\end{algorithm}
\fi
\ifmefomo
\else
Convolution-specific hyperparameters are shown in Table~\ref{tab:hyperparams}.
Compared to the hyperparameters necessary to train S4, our regularization approaches have fewer hyperparameters and choices than S4.
\fi

\subsection{\fastconv}
\label{sec:fastconv}

\ifmefomo
We present \fastconv, an IO-aware algorithm for speeding up general convolutions on modern hardware.
Following H3~\citep{dao2022hungry}, we use kernel fusion to reduce GPU memory IO requirements, and use a Butterfly decomposition to rewrite the FFT as a series of block-sparse matrix multiplications, allowing better utilization of modern matrix multiply units.
The details are shown in Appendix~\ref{sec:methods_details_supp}.
To scale to sequences that does not fit into SRAM (length 8K or longer on A100), the method presented in H3~\citep{dao2022hungry} does not work anymore, as it depended in a critial way on the recurrent nature of convolutions induced by SSMs.  
Instead, we use an alternate Butterfly decomposition to construct a three-pass FFT convolution algorithm to further reduce IO requirements.
\else
In addition to improving the quality of long convolutions, it is also critical to improve runtime performance.
We present \fastconv, an IO-aware algorithm for speeding up general convolutions on modern hardware.
We use kernel fusion to reduce GPU memory IO requirements, and use a Butterfly decomposition to rewrite the FFT as a series of block-sparse matrix multiplications.
To scale to long sequences, we use an alternate Butterfly decomposition to construct a three-pass FFT convolution algorithm to further reduce IO requirements.
\fi

\ifmefomo
\else
\ifarxiv
\paragraph{Kernel Fusion}
\else
\textbf{Kernel Fusion.}
\fi
Naive implementations of the FFT convolution incur expensive GPU memory IO.
The FFT, inverse FFT, and pointwise multiplication in Equation~\ref{eq:conv_fft} each require at least one read and write of the input sequence from GPU memory.
For long sequences, the IO costs may be even worse: the entire input sequence cannot fit into SRAM, so optimized implementations such as cuFFT~\citep{cufft} must take multiple passes over the input sequence using the Cooley-Tukey decomposition of the FFT~\citep{cooley1965an}.
Following \textsc{FlashAttention}~\citep{dao2022flashattention}, \fastconv's first fuses the entire FFT convolution into a single kernel to compute the entire convolution in GPU SRAM and avoid this overhead.

\ifarxiv
\paragraph{Butterfly Decomposition}
\else
\textbf{Butterfly Decomposition.}
\fi
Kernel fusion reduces the IO requirements, but the fused FFT operations still cannot take full advantage of specialized matrix multiply units on modern GPUs, such as Tensor Cores on Nvidia GPUs, which perform fast 16 $\times$ 16 matrix multiplication.
We appeal to a classical result, known as the four-step or six-step FFT algorithm~\citep{bailey1990ffts}, that rewrites the FFT as a series of block-diagonal Butterfly matrices~\citep{parker1995random} interleaved with permutation.

The Butterfly decomposition states that we can decompose an $N$-point FFT into a series of FFTs of sizes $N_1$ and $N_2$, where $N = N_1N_2$.
Conceptually, the algorithm reshapes the input as an $N_1 \times N_2$ matrix, applies $N_1$ FFTs of size $N_2$ to the columns, multiplies each element by a twiddle factor, and then applies $N_2$ FFTs of size $N_1$ to the rows.

More precisely, let $\vF_N$ denote the DFT matrix corresponding to taking the $N$-point FFT.
Then, there exist permutation matrices $\vP$, and a diagonal matrix $\vD$, such that $\vF_N = \vP(\vI_{N_2} \otimes \vF_{N_1})\vP^T \vD(\vI_{N_1} \otimes \vF_{N_2})\vP$.
$\vP$ denotes a permutation matrix that reshapes the input to $N_1 \times N_2$ and takes the transpose, $\vD$ denotes a diagonal matrix with the twiddle factors along the diagonal, $\otimes$ denotes the Kronecker product, and $vI_{N_i}$ and $\vF_{N_i}$ are the identity and DFT matrices of size $N_i \times N_i$.
Precise values for $\vF_{N_i}$, $\vD$, and $\vP$ are given in Appendix~\ref{sec:methods_details_supp}.

The Butterfly decomposition incurs $O(Nr \log N / \log r)$ FLOPS for a sequence length $N = r^p$, with \textit{block size} $r$.
In general FFT implementations, $N$ is typically padded to a power of two, so that the block size can be set to $2$ to minimize the total number of FLOPS.
However, on GPUs with a specialized $b \times b$ matrix multiply unit, the FLOP cost of computing an $r \times r$ matrix multiply with $r < b$ is equivalent to performing a single $b \times b$ matrix multiply.
Thus, the actual FLOP count scales as $O(N b \log N / \log r)$ for $r < b$.
Increasing the block size up to $b$ actually \textit{reduces} the FLOP cost.

Table~\ref{tab:block_micro} demonstrates this tradeoff on an A100 GPU, which has specialized matrix multiply units up to 16 $\times$ 32.
Runtime decreases as $r$ increases from 2, even though theoretical FLOPS increase.
Once $r > b$, runtime begins increasing as actual FLOPS increase as well.

\begin{table}[t]
    \caption{Runtime, GLOPs, and FLOP util for the Butterfly decomposition with different block sizes $r$ for sequence length 4096, on A100 with batch size $128$, head dimension $32$.}
    \label{tab:block_micro}
    \ifarxiv
    \else
    \small
    \fi
    \centering
    \begin{tabular}{rccc} \toprule
    {Block Size} & {Runtime (ms)} & GLOPs & FLOP Util \\ \midrule
    2 & 0.52 & 2.0 & 1.3\% \\
    16 & 0.43 & 8.1 & 6.0\% \\
    64 & 0.53 & 21.5 & 13.0\% \\
    256 & 0.68 & 64.5 & 30.4\% \\
    \bottomrule 
    \end{tabular}
\end{table}

\fi

\ifarxiv
\paragraph{Three-Pass Algorithm}
\else
\textbf{Three-Pass Algorithm.}
\fi
\ifmefomo
\else
Kernel fusion and the Butterfly decomposition improve runtime performance, but only for convolutions short enough to fit into SRAM (length 8K or shorter on A100).
For longer sequences, we again appeal to the Butterfly decomposition, but using an alternate formulation that eliminates permutations over the input sequence.
This formulation allows us to decompose the convolution into three passes over the data: a Butterfly matrix multiplication that can be computed with a single IO, FFT convolutions that we can compute in parallel, and a final Butterfly matrix multiplication that can also be computed with a single IO.

\fi
In particular, we rewrite the DFT matrix $\vF_N$ of size $N$ as $N \vP^{-1} (\vI_{m} \otimes (l \vF_{l})) \overline{\vB}^{-1},$ and its inverse matrix $\vF_N^{-1}$ as $N^{-1} \overline{\vB} (\vI_{m} \otimes \overline{\vF}_{l}) \vP$, where $\vB$ is an $N \times N$ block matrix with $m^2$ blocks of size $l \times l$, each of which is diagonal (see Appendix~\ref{sec:methods_details_supp} for the exact derivation).
Critically, matrix-vector multiply $\vB u$ can be computed in a single pass over the input vector $u$.
\ifmefomo
Substituting these into the FFT convolution decomposition and simplifying yields the following:
\else
Substituting these into Equation~\ref{eq:conv_fft} and simplifying yields the following:
\fi
\ifmefomo
$y = u \ast \vK = \overline{\vB} (\vI_{m} \otimes \overline{\vF}_{l}) \vD'_{\vK} (\vI_{m} \otimes \vF_{l}) \overline{\vB}^{-1}$,
\else
\begin{equation}
\label{eq:three_pass}
y = u \ast \vK = \overline{\vB} (\vI_{m} \otimes \overline{\vF}_{l}) \vD'_{\vK} (\vI_{m} \otimes \vF_{l}) \overline{\vB}^{-1},
\end{equation}
\fi
where $\vD'_{\vK} = l\vP\vD_{\vK}\vP^{-1}$ is another diagonal matrix.
The middle terms can now be computed as $m$ independent FFT convolutions of size $l$, with a different convolution kernel.
These parallel convolutions collectively require one pass over $N$ input elements, so the entire convolution can be computed with three passes over the input.

\ifmefomo
\else
The full algorithm for \fastconv for $N > l$ is shown in Algorithm~\ref{alg:fastconv}.
\ifarxiv
\begin{algorithm}[h]
\else
\begin{algorithm}[t]
\fi
    \algsetup{linenosize=\tiny}
    \caption{\label{alg:fastconv} \fastconv}
    \small
    \begin{algorithmic}[1]
      \REQUIRE Input $u \in \mathbb{R}^{B \times H \times N}$, $\vK \in \mathbb{R}^{H \times N}$, $\vD \in \mathbb{R}^{H}$, where $N = lm$ is the sequence length, $H$ is the head dimension, and $B$ is the batch size. 
      \STATE $\hat{\vK} \gets FFT(\vK)$
      \STATE $\vD'_{\vK} \gets \vP (\hat{\vK}) \vP^{-1}$ 
      \STATE $u \gets \overline{\vB}^{-1}u$
      \STATE Compute $u \gets (\vI_{m} \otimes \overline{\vF}_{l}) \vD'_{\vK} (\vI_{m} \otimes \vF_{l}) u$
      in parallel across $m$ streaming multiprocessors
      \STATE Return $\overline{\vB} u \in \bR^{B \times H \times N}$
    \end{algorithmic}
  \end{algorithm}

We show that Algorithm~\ref{alg:fastconv} is correct, and that it can be computed in three passes over the input sequence.
The proof is given in Appendix~\ref{sec:theory_supp}.
\begin{proposition}
    \label{prop:fastconv}
    Algorithm~\ref{alg:fastconv} computes the convolution $u \ast \vK$ with at most three passes over the input sequence $u$.
\end{proposition}
\fi
\ifmefomo
\else
\ifarxiv
\subsubsection{Learned Butterfly Extension\label{sec:method_extension}}
\else
\textbf{Learned Butterfly Extension.}
\fi
The Butterfly decomposition in \fastconv suggests a natural extension: learning the values of the Butterfly matrices $\vF_r$ in the Butterfly decomposition, instead of using the fixed matrices corresponding to the FFT.
If we keep the block size $r$ fixed, then the number of parameters in the Butterfly matrices increases by $O(Hr^2)$, but the total FLOPS in the model stay the same.
Increasing the block size allows us to further increase expressivity, but at additional compute cost.
As $r$ approaches $N$, the Butterfly decomposition approaches the compute cost and expressivity of a full dense matrix multiply: $O(N^2)$.
\fi

\section{Evaluation\label{sec:eval}}

\ifmefomo
We evaluate how well long convolutions perform in the challenging LRA benchmark as well as on the OpenWebText language task.
\else
We evaluate how well long convolutions perform in a variety of challenging sequence modeling tasks from diverse modalities and benchmarks, including the long range arena benchmark, image classification, text modeling, and brain data modeling (Section~\ref{subsec:quality}).
We find that long convolutions are strong sequence modelers across these tasks.
\fi
Next, we evaluate the runtime performance of long convolutions under \fastconv and evaluate how well it scales to very long sequences (Section~\ref{subsec:efficiency}).
\ifmefomo
\else
Finally, we evaluate the quality improvements from learned Butterfly extension (Section~\ref{subsec:exp_extension}).
\fi

\ifmefomo
\begin{table*}[t]
    \caption{Validation accuracy of different models on the LRA benchmark. Best in bold, second best underlined.}
    \label{tab:lra}
    \scriptsize
    \centering
    \begin{tabular}{rccccccc} \toprule
    {Model} & {ListOps} & {Text} & {Retrieval} & {Image} & {Pathfinder} & {Path-X} & {Avg} \\ \midrule
    Transformer  & 36.4  & 64.3 & 57.5  & 42.4 & 71.4 & \xmark  & 53.7 \\
    \ifmefomo
    \else
    Nyströmformer & 37.2   & 65.5 & 79.6   & 41.6 & 70.9 & \xmark  & 57.5  \\
    Reformer & 37.3 & 56.1 & 53.4 & 38.1 & 68.5 & \xmark & 50.6 \\
    BigBird & 36.1 & 64.0 & 59.3 & 40.8 & 74.9 & \xmark & 54.2 \\
    Linear Trans. & 16.1 & 65.9 & 53.1 & 42.3 & 75.3 & \xmark & 50.5 \\
    Performer & 18.0 & 65.4 & 53.8 & 42.8 & 77.1 & \xmark & 51.2 \\
    \fi
    \midrule
    S4-LegS  & 59.6  & 86.8 & \underline{90.9}  & \underline{88.7} & 94.2 & \underline{96.4}  & 86.1  \\ 
    S4-FouT  & 57.9  & 86.2 & 89.7  & \textbf{89.1} & \underline{94.5} & \xmark  & 77.9  \\
    \ifmefomo
    \else
    S4-LegS/FouT & 60.5 & 86.8 & 90.3 & \underline{89.0} & 94.4 & \xmark & 78.5 \\
    S4D-LegS & 60.5 & 86.2 & 89.5 & 88.2 & 93.1 & 92.0 & 84.9 \\ 
    S4D-Inv & 60.2 & 87.3 & \underline{91.1} & 87.8 & 93.8 & 92.8 & 85.5 \\
    S4D-Lin & 60.5 & 87.0 & 91.0 & 87.9 & 94.0 & \xmark & 78.4 \\
    \fi
    S4 (Original) & 58.4 & 76.0 & 87.1 & 87.3 & 86.1 & 88.1 & 80.5 \\ \midrule
    Long Convs, No Regularization & 53.4 & 64.4 & 83.0 & 81.4 & 85.0 & \xmark & 69.5 \\
    Long Convs, Random Init + \squash & \underline{61.4} & \underline{88.0} & 90.2 & \underline{88.7} & \textbf{94.6} & \textbf{97.1} & \textbf{86.7} \\
    Long Convs, Geom Init + \squash & \textbf{62.2} & \textbf{89.6} & \textbf{91.3} & 87.0 & 93.2 & 96.0 & \underline{86.6} \\
    \bottomrule 
    \end{tabular}
\vspace{-2em}
\end{table*}
\else
\begin{table*}[t]
    \caption{Validation accuracy of different models on the LRA benchmark. Best in bold, second best underlined.}
    \label{tab:lra}
    \small
    \centering
    \begin{tabular}{rccccccc} \toprule
    {Model} & {ListOps} & {Text} & {Retrieval} & {Image} & {Pathfinder} & {Path-X} & {Avg} \\ \midrule
    Transformer  & 36.4  & 64.3 & 57.5  & 42.4 & 71.4 & \xmark  & 53.7 \\
    Nyströmformer & 37.2   & 65.5 & 79.6   & 41.6 & 70.9 & \xmark  & 57.5  \\
    Reformer & 37.3 & 56.1 & 53.4 & 38.1 & 68.5 & \xmark & 50.6 \\
    BigBird & 36.1 & 64.0 & 59.3 & 40.8 & 74.9 & \xmark & 54.2 \\
    Linear Trans. & 16.1 & 65.9 & 53.1 & 42.3 & 75.3 & \xmark & 50.5 \\
    Performer & 18.0 & 65.4 & 53.8 & 42.8 & 77.1 & \xmark & 51.2 \\
    \midrule
    S4-LegS  & 59.6  & 86.8 & 90.9  & 88.7 & \underline{94.2} & \underline{96.4}  & \underline{86.1}  \\ 
    S4-FouT  & 57.9  & 86.2 & 89.7  & \textbf{89.1} & \textbf{94.5} & \xmark  & 77.9  \\
    S4-LegS/FouT & \underline{60.5} & 86.8 & 90.3 & \underline{89.0} & 94.4 & \xmark & 78.5 \\
    S4D-LegS & \underline{60.5} & 86.2 & 89.5 & 88.2 & 93.1 & 92.0 & 84.9 \\ 
    S4D-Inv & 60.2 & \underline{87.3} & \underline{91.1} & 87.8 & 93.8 & 92.8 & 85.5 \\
    S4D-Lin & 60.5 & 87.0 & 91.0 & 87.9 & 94.0 & \xmark & 78.4 \\
    S4 (Original) & 58.4 & 76.0 & 87.1 & 87.3 & 86.1 & 88.1 & 80.5 \\ \midrule
    Long Conv, Rand & 53.4 & 64.4 & 83.0 & 81.4 & 85.0 & \xmark & 69.5 \\
    Long Conv, Rand + \smooth & 59.8 & 68.7 & 86.6 & 79.3 & 86.1 & \xmark & 71.8 \\
    Long Conv, Rand + \squash & 60.3 & 87.1 & 90.0 & 88.3 & 94.0 & \textbf{96.9} & \underline{86.1} \\
    Long Conv, Rand + \squash + \smooth & 59.7 & 72.8 & 88.6 & 80.8 & 90.1 & \xmark & 73.7 \\
    Long Conv, Exp + \squash & \textbf{62.2} & \textbf{89.6} & \textbf{91.3} & 87.0 & 93.2 & 96.0 & \textbf{86.6} \\
    \bottomrule 
    \end{tabular}
\ifarxiv
\else
\vspace{-2em}
\fi
\end{table*}
\fi

\ifarxiv
\begin{table}[t]
\begin{minipage}{.5\linewidth}
    \caption{Image classification on flattened images.}
    \label{tab:scifar}
    \centering
    \small
    \begin{tabular}{rc} \toprule
    {Model} & {sCIFAR} \\ \midrule
    Transformer & 62.2 \\ \midrule
    LSTM & 63.0 \\
    r-LSTM & 72.2 \\
    UR-LSTM & 71.0 \\
    UR-GRU & 74.4 \\
    HIPPO-RNN & 61.1 \\
    LipschitzRNN & 64.2 \\ 
    CKConv & 64.2 \\ \midrule
    S4-LegS & \underline{91.8} \\
    S4-FouT & 91.2 \\
    S4D-LegS & 89.9 \\
    S4D-Inv & 90.7 \\
    S4D-Lin & 90.4 \\
    \midrule
    Long Conv, Random  & 91.0 \\
    Long Conv, Geom Init  & \textbf{92.1} \\
    \bottomrule 
    \end{tabular}
\ifarxiv
\else
\vspace{-1em}
\fi
\end{minipage}
\begin{minipage}{.5\linewidth}
    \caption{Image classification on 2D images.}
    \label{tab:s42d}
    \centering
    \small
    \begin{tabular}{rc} \toprule
    {Model} & {CIFAR} \\ \midrule
    S4ND-ISO & \textbf{89.9} \\ \midrule
    Long Conv 2D-ISO, Rand init  & 88.1 \\
    Long Conv 2D-ISO, Geom init  & \underline{89.1} \\
    \bottomrule 
    \end{tabular}
    \ifarxiv
    \else
    \vspace{-1em}
    \fi
\end{minipage}
\end{table}

\begin{table}[t]
\begin{minipage}{.45\linewidth}
    \caption{Test PPL of models trained on OpenWebText.}
    \label{tab:owt}
    \centering
    \small
    \begin{tabular}{rc} \toprule
    {Model} & {Test PPL} \\ \midrule
    Transformer & 20.6 \\
    S4D & 24.9 \\
    GSS & 24.0 \\
    H3 & \textbf{19.6} \\ \midrule
    H3 + Long-Conv, Rand Init & 20.1 \\
    H3 + Long-Conv, Geom Init & \underline{19.9} \\
    \bottomrule 
    \end{tabular}
\ifarxiv
\else
\vspace{-1em}
\fi
\end{minipage}~~
\begin{minipage}{.45\linewidth}
    \caption{\label{tab:the_pile} Test PPL on the Pile for models trained with various tokens.}
    \centering
    \small
    \begin{tabular}{rccc}
    \toprule
    Train Tokens & 5B & 10B & 15B \\
    \hline
    Transformer & 12.7 & 11.3 & 10.7 \\
    H3 & \textbf{11.8} & \textbf{10.7} & \textbf{10.2} \\ \hline
    H3 + Long Convs, Geom Init & \underline{11.9} & \textbf{10.7} & \underline{10.3} \\
    \bottomrule
    \end{tabular}
\end{minipage}
\end{table}

\begin{table}[t]
    \caption{Evaluation on brain fMRI data.}
    \label{tab:fmri}
    \centering
    \small
    \begin{tabular}{rc} \toprule
    {Model} & {MAE} \\
    \midrule
    {Transformer} & {0.68} \\ 
    {H3} & {0.70} \\
    {H3 + Long Convs, Rand Init} & \underline{0.58} \\
    {H3 + Long Convs, Geom Init} & \bf{0.54} \\
    \bottomrule 
    \end{tabular}
\ifarxiv
\else
\vspace{-1em}
\fi
\end{table}
\else
\ifmefomo
\begin{table}[t]
    \begin{minipage}{.5\linewidth}
    \caption{Test PPL of models trained on OpenWebText.}
    \label{tab:owt1}
    \centering
    \scriptsize
    \begin{tabular}{rc} \toprule
    {Model} & {Test PPL} \\ \midrule
    Transformer & 20.6 \\
    \ifmefomo
    \else
    S4D & 24.9 \\
    \fi
    GSS & 24.0 \\
    H3 & \textbf{19.6} \\ \midrule
    H3 + Long-Conv, Rand Init & 20.1 \\
    H3 + Long-Conv, Exp Init & 19.9 \\
    \bottomrule 
    \end{tabular}
    \end{minipage}
    \begin{minipage}{.5\linewidth}
    \caption{LRA Speed Benchmark.}
    \label{tab:lra_speed1}
    \centering
    \scriptsize
    \begin{tabular}{rc} \toprule
    {Model} & {Speedup} \\ \midrule
    Transformer & 1$\times$ \\
    FlashAttention & 2.4$\times$ \\
    Block-Sparse FlashAttention & 2.8$\times$ \\ \midrule
    S4 & 2.9$\times$ \\ \midrule
    \fastconv & \textbf{\num{7.0}$\times$} \\
    \bottomrule 
    \end{tabular}
    \end{minipage}
\end{table}
\else
\begin{table}[ht]
    \caption{Image classification on flattened images.}
    \label{tab:scifar}
    \centering
    \small
    \begin{tabular}{rc} \toprule
    {Model} & {sCIFAR} \\ \midrule
    Transformer & 62.2 \\ \midrule
    LSTM & 63.0 \\
    r-LSTM & 72.2 \\
    UR-LSTM & 71.0 \\
    UR-GRU & 74.4 \\
    HIPPO-RNN & 61.1 \\
    LipschitzRNN & 64.2 \\ 
    CKConv & 64.2 \\ \midrule
    S4-LegS & \underline{91.8} \\
    S4-FouT & 91.2 \\
    S4D-LegS & 89.9 \\
    S4D-Inv & 90.7 \\
    S4D-Lin & 90.4 \\
    \midrule
    Long Conv, Random  & 91.0 \\
    Long Conv, Geom Init  & \textbf{92.1} \\
    \bottomrule 
    \end{tabular}
\ifarxiv
\else
\vspace{-1em}
\fi
\end{table}
\begin{table}[ht]
    \caption{Image classification on 2D images.}
    \label{tab:s42d}
    \centering
    \small
    \begin{tabular}{rc} \toprule
    {Model} & {CIFAR} \\ \midrule
    S4ND-ISO & \textbf{89.9} \\ \midrule
    Long Conv 2D-ISO, Rand init  & 88.1 \\
    Long Conv 2D-ISO, Geom init  & \underline{89.1} \\
    \bottomrule 
    \end{tabular}
    \ifarxiv
    \else
    \vspace{-1em}
    \fi
\end{table}
\begin{table}[t]
    \caption{Evaluation on brain fMRI data.}
    \label{tab:fmri}
    \centering
    \small
    \begin{tabular}{rc} \toprule
    {Model} & {MAE} \\
    \midrule
    {Transformer} & {0.68} \\ 
    {H3} & {0.70} \\
    {H3 + Long Convs, Rand Init} & \underline{0.58} \\
    {H3 + Long Convs, Geom Init} & \bf{0.54} \\
    \bottomrule 
    \end{tabular}
\ifarxiv
\else
\vspace{-1em}
\fi
\end{table}
\begin{table}[t]
    \caption{Test PPL of models trained on OpenWebText.}
    \label{tab:owt}
    \centering
    \small
    \begin{tabular}{rc} \toprule
    {Model} & {Test PPL} \\ \midrule
    Transformer & 20.6 \\
    S4D & 24.9 \\
    GSS & 24.0 \\
    H3 & \textbf{19.6} \\ \midrule
    H3 + Long-Conv, Rand Init & 20.1 \\
    H3 + Long-Conv, Geom Init & \underline{19.9} \\
    \bottomrule 
    \end{tabular}
\ifarxiv
\else
\vspace{-1em}
\fi
\end{table}
\fi
\fi

\subsection{Quality on Sequence Modeling\label{subsec:quality}}

\ifmefomo
We begin by evaluating various regularization and initialization techniques on the long range arena benchmark, a suite of six general-purpose sequence modeling tasks with sequence length between 1K and 16K tokens, covering
modalities including text, natural and synthetic images, and mathematical expressions~\citep{tay2020long}.
We then evaluate long convolutions on language modeling.
\else
In this section, we evaluate the performance of long convolutions in sequence modeling in terms of quality.
We begin by evaluating various regularization and initialization techniques on the long range arena benchmark, a suite of general-purpose sequence modeling tasks designed to stress test long sequences~\citep{tay2020long}.
We take the best-performing variants and move on to two challenging and diverse modalities that have been used to evaluate sequence models, including SSMs: image classification (both one-dimensional and two-dimensional) and text modeling.
We conclude the section with a real-world application of long convolutions to brain data modeling.

We find that long convolutions perform well across all of these diverse tasks and modalities---and are generally more robust to choice of initialization than SSMs.
Our results suggest that long convolutions may be a compelling simpler alternative to SSMs for sequence modeling.
\fi
Experimental details for the tasks are given in Appendix~\ref{sec:exp_details_supp}, and additional experiments are provided in Appendix~\ref{sec:exp_supp}.

\ifarxiv
\subsubsection{Long Sequence Modeling: Long Range Arena}
\else
\textbf{Long Sequence Modeling: Long Range Arena.}
\fi
\ifmefomo
\else
We first evaluate long convolutions on Long Range Arena (LRA), a benchmark suite used to test general-purpose sequence modeling over long contexts.
LRA consists of six long-range sequence modeling tasks, with sequence lengths between 1K and 16K tokens.
The tasks have modalities including text, natural and synthetic images, and mathematical expressions.
We take the state-of-the-art S4 architecture~\citep{gu2022train}, and replace the SSM layers with long convolutions.

We present five variants of long convolutions: random intialization and no regularization, random initialization with the \smooth operator, random initialization with the \squash operator, random initialization with both operators, and the geometric initialization with the \squash operator.
We compare the long convolution methods against variants of Transformers presented in the original Long Range Arena paper~\citep{tay2020long}, as well as variants of S4 with different parameterizations and initializations~\citep{gu2022train}.
These initializations are important for S4 to achieve high quality.

\fi
Table~\ref{tab:lra} shows the results for long convolutions on the LRA benchmark.
An \xmark\xspace in the Path-X column indicates that the model never achieved better classification accuracy than random guessing.
\ifmefomo
Long convolutions are more robust to different initializations than variants of S4.
\else
Long convolutions appear to be robust to initialization: there is only a \num{0.5} point spread in the average score between long convolutions with a geometric initialization and long convolutions with a random initialization---though individual tasks may have more spread.
This stands in contrast to the S4 methods, which are sensitive to initialization choices and the parameterization---with a spread of \num{7.6} points between S4-LegS and S4-LegS/FouT.
\fi

Regularization is critical for achieving strong performance; without it, long convolutions lose \num{17.1} points on average across the six LRA tasks.
Using the \squash operator on its own appears to perform better than using the \smooth operator, or using both together.
For the rest of the experiments, we focus on the two best-performing variants of long convolutions: random initialization with the \squash operator, and geometric initialization with the \squash operator.

\ifmefomo
\else
\ifarxiv
\subsubsection{Image Classification}
\else
\textbf{Image Classification}
\fi
Next, we evaluate long convolutions on image classification.
We evaluate two settings which have been used to evaluate SSMs and sequence models: 1D pixel-by-pixel image classification, and 2D image classification.
These settings are challenging for sequence modeling, as they require modeling complex spatial relationships between image pixels in a continuous space.
For the 1D case, we again use long convolutions as a drop-in replacement for the SSM layer in the state-of-the-art S4 architecture.
For the 2D case, we replace the S4 layers in S4ND~\citep{nguyen2022s4nd} with 2D long convolution filters.

Tables~\ref{tab:scifar} and~\ref{tab:s42d} show the results.
On 1D image classification, long convolutions again match the performance of S4, even with random initializations, while their performance improves further by \num{1.1} points when using the geometric initialization.
On 2D image classification, long convolutions come within \num{0.8} points of the state-of-the-art S4ND model.
Further regularization or inductive bias may be helpful for long convolutions to recover the performance of SSMs in higher dimensions.
\fi

\ifarxiv
\subsubsection{Text Modeling: OpenWebText and the PILE}
\else
\textbf{Text Modeling: OpenWebText and the PILE.}
\fi
\ifmefomo
We evaluate long convolutions on text modeling.
\else
We evaluate long convolutions on text modeling.
Text has been a challenging modality for state space models and non-attention sequence models, since it requires comparing and copying elements across the input sequence~\citep{dao2022hungry,olsson2022context}.
\fi
We build off of the H3 model~\citep{dao2022hungry}---the state-of-the-art SSM model for text modeling---which stacks two SSMs and multiplies their outputs together as a gating mechanism.
We use long convolutions as a drop-in replacement for the SSMs in the H3 layer.

Following the H3 paper, we keep two attention layers in the overall language model and evaluate on two datasets: OpenWebText~\citep{Gokaslan2019OpenWeb} and the Pile~\citep{gao2020pile}.
We use OpenWebText to evaluate the role of initialization: we train models to completion at 100B tokens, and evaluate both random and geometric initializations.
For the Pile, we evaluate how well long convolutions scale with data: we use the geometric initialization, and evaluate the performance of models trained with 5B, 10B, and 15B tokens.

\ifmefomo
Table~\ref{tab:owt1} shows the results.
\else
Tables~\ref{tab:owt} and~\ref{tab:the_pile} show the results.
\fi
On OpenWebText, long convolutions with random initialization come within \num{0.5} PPL points of H3, and the geometric decay initialization comes within \num{0.3} PPL.
Both models outperform the Transformer.
On the Pile, long convolutions with geometric decay initialization nearly match H3 everywhere along the data scaling curve, and outperform Transformers.
These initial results suggests that convolutions---with some multiplicative gating mechanism---may be a promising candidate for language modeling.

\ifmefomo
\else
\ifarxiv
\subsubsection{Brain fMRI Analysis}
\else
\textbf{Brain fMRI Analysis.}
\fi
Finally, we evaluate long convolutions on a real-world sequence modeling modality: analysis of brain functional Magnetic Resonance Imaging (fMRI) sequence data.
To this end, we replicate the self-supervised pre-training task proposed by~\citet{thomas2022self}: 
training models to predict whole-brain activity for the next time step of an fMRI sequence (using a large-scale upstream dataset, spanning fMRI data from $11,980$ experimental runs of $1,726$ individuals).
We compare long convolutions against Transformers and H3, architectures that achieve state-of-the-art performance in this task~\citep{thomas2022self,dao2022hungry}, by adapting the H3 model and replacing the SSM kernel with long convolutions.
Long convolutions outperform the other models in accurately predicting brain activity in this task (see Table~\ref{tab:fmri}).
Full details of this analysis are provided in Appendix~\ref{sec:fmri_supp}, where we also show that long convolutions perform on par with the other models in accurately classifying new fMRI sequences in a downstream adaptation.
\fi

\subsection{Efficiency: \fastconv\label{subsec:efficiency}}

\ifmefomo
\else
\begin{table}[t]
    \caption{LRA Speed Benchmark.}
    \label{tab:lra_speed}
    \centering
    \ifarxiv
    \else
    \small
    \fi
    \begin{tabular}{rc} \toprule
    {Model} & {Speedup} \\ \midrule
    Transformer & 1$\times$ \\
    \textsc{FlashAttention} & 2.4$\times$ \\
    SSM + \textsc{FlashConv} & 5.8$\times$ \\ \midrule
    \fastconv & \textbf{\num{7.0}$\times$} \\
    \bottomrule 
    \end{tabular}
\ifarxiv
\else
\vspace{-1em}
\fi
\end{table}
\begin{figure*}[t]
    \centering
    \ifarxiv
    \includegraphics[width=\textwidth]{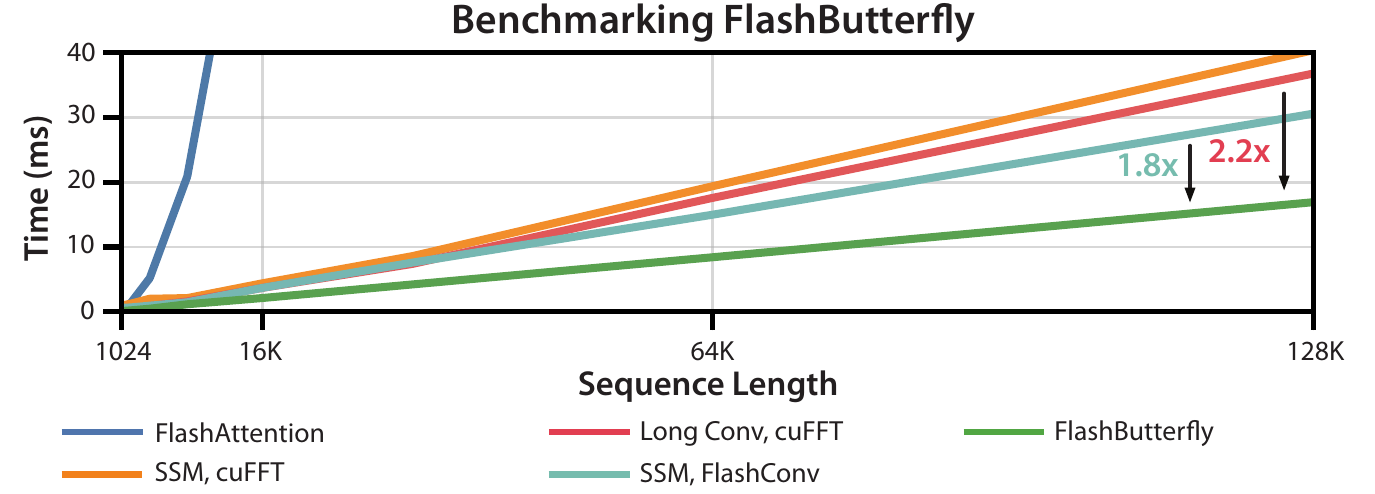}
    \else
    \includegraphics[width=5in]{figs/performance_pdf.pdf}
    \vspace{-1em}
    \fi
    \caption{\label{fig:fast_conv}
    We compare the performance of \fastconv to attention, SSMs with cuFFT, long convolutions with cuFFT, and SSMs with \textsc{FlashConv}, the most optimized SSM algorithm~\citep{dao2022hungry}.
    Speedups shown for sequence length 128K.
    }
\ifarxiv
\else
\vspace{-1em}
\fi
\end{figure*}
\fi
\begin{table}[t]
    \caption{Runtime and accuracy on Path256 (sequence length 64K).}
    \label{tab:path256}
    \centering
    \ifarxiv
    \else
    \scriptsize
    \fi
    \begin{tabular}{rcc} \toprule
    {Model} & {Accuracy} & {Training Time} \\ \midrule
    Transformer & \xmark & \xmark \\ \midrule
    \textsc{FlashAttention} & \xmark & \xmark \\
    Block-Sparse \textsc{FlashAttention} & 63.1 & 3 days \\ \midrule
    \fastconv & \textbf{92.2} & 10 hours \\
    \bottomrule 
    \end{tabular}
\ifarxiv
\else
\vspace{-2em}
\fi
\end{table}

\ifmefomo
\else
We now turn towards evaluating the runtime performance of \fastconv.
We focus on two questions: whether \fastconv can outperform SSMs in terms of runtime performance, and how well \fastconv can scale to long sequences.
First, we evaluate \fastconv's runtime on the Long Range Arena speed benchmark~\citep{tay2020long}, which measures runtime on a byte-level text classification benchmark that is representative of standard sequence modeling loads.
\fastconv outperforms SSMs and baselines from the original LRA speed benchmark.
Next, we evaluate how well \fastconv scales to longer sequences.
Across many sequence lengths, \fastconv outperforms the fastest SSM implementation.
Finally, we demonstrate \fastconv's sequence scaling capabilities on an extremely long sequence task: Path256, which has sequence length 64K.
\fi

\ifmefomo
\textbf{Runtime on Long Range Arena.}
\else
\ifarxiv
\subsubsection{Runtime on Long Range Arena}
\else
\textbf{Runtime on Long Range Arena.}
\fi
\fi
We begin by evaluating runtime on the Long Range Arena speed benchmark~\citep{tay2020long}.
The benchmark measures runtime on a byte-level text classification task.
This task, which has sequence length 4K, is representative of typical sequence modeling training workloads, and is a standard evaluation benchmark for Transformers and SSMs~\citep{tay2020long}.
The benchmark is measured in terms of speedup against vanilla Transformers using a HuggingFace implementation.
We additionally compare against two more baselines: a) Transformers using \textsc{FlashAttention}~\citep{dao2022flashattention}, the fastest attention algorithm, and b) SSMs using \textsc{FlashConv}~\citep{dao2022hungry}, the fastest SSM implementation.

\ifmefomo
Table~\ref{tab:lra_speed1} compares a long convolution with \fastconv against Transformers, and S4 with \textsc{FlashConv}.
\else
Table~\ref{tab:lra_speed} shows the results.
\fi
\fastconv achieves \num{7.0$\times$} speedup over the Transformer baseline.
It outperforms \textsc{FlashAttention}, since its compute scales nearly linearly with sequence length instead of quadratically.
It also outperforms \textsc{FlashConv}, the fastest SSM implementation, since it does not require kernel generation.
These results show that \fastconv outperforms SSMs and Transformers in terms of runtime efficiency in standard sequence modeling workloads.

\ifmefomo
\else
\ifarxiv
\subsubsection{Scaling to Longer Sequences}
\else
\textbf{Benchmark Across Sequence Lengths.}
\fi
Next, we evaluate how well \fastconv scales to longer sequence lengths.
We compare \fastconv against a) convolutions using cuFFT, the standard implementation in PyTorch, and b) SSMs using \textsc{FlashConv}.
We measure the runtime for sequence lengths ranging from \num{1K} to \num{128K}.
Following~\citep{dao2022hungry}, we measure the runtime of a single layer using batch size \num{32} and \num{128} model dimension.
We also provide attention runtime, as well as SSMs using a standard PyTorch implementation, for context.

Figure~\ref{fig:fast_conv} shows the results.
\fastconv yields up to \num{2.2$\times$} speedup against baseline cuFFT-based convolutions.
\fastconv outperforms \textsc{FlashConv} for all sequence lengths, since it does not require the kernel generation step of SSMs.
These results show that \fastconv outperforms SSMs and Transformers across all sequence lengths---even very long sequences.
\fi

\ifarxiv
\else
\textbf{Very Long Sequence Lengths.}
\fi
We demonstrate the utility of \fastconv by training models on a task with extremely long sequences: Path256, which has sequence length 64K.
Table~\ref{tab:path256} shows that long convolutions achieve state-of-the-art performance on Path256, outperforming block-sparse \textsc{FlashAttention} from~\citep{dao2022flashattention}, the only prior work to report non-trivial performance ($>$50\% accuracy) on Path256.
Long convolutions with \fastconv exceed state-of-the-art performance by \num{29.1} points, and train \num{7.2$\times$} faster.

\ifmefomo
\else
\subsection{Learned Butterfly Extension\label{subsec:exp_extension}}

\ifarxiv
\begin{table}[t]
    \begin{minipage}{.45\linewidth}
    \caption{Performance with learnable Butterfly of different sizes}
    \label{tab:monarch}
    \centering
    \ifarxiv
    \else
    \small
    \fi
    \begin{tabular}{rcc} \toprule
    Block Size & {sCIFAR} & {Speedup} \\ \midrule
    Fixed Butterfly & 91.0 & 1$\times$\\ \midrule
    16 & 91.8 & 1$\times$\\
    32 & 92.4 & 0.9$\times$\\
    256 & \textbf{92.5} & 0.6$\times$ \\
    \bottomrule 
    \end{tabular}
\ifarxiv
\else
\vspace{-2em}
\fi
    \end{minipage}~~
    \begin{minipage}{.5\linewidth}
    \caption{Performance of replacing MLPs with the long conv extension in a Transformer on WikiText103.}
    \label{tab:wikitext}
    \centering
    \ifarxiv
    \else
    \small
    \fi
    \begin{tabular}{rcc} \toprule
    {Model} & {PPL} & Params  \\ \midrule
    GPT-2-Small & 20.6 & 124M \\
    Monarch-GPT-2-Small & 20.7 & 72M  \\ \midrule
    \fastconv-GPT-2-Small & \textbf{20.4} & 86M  \\
    \bottomrule 
    \end{tabular}
\ifarxiv
\else
\vspace{-2em}
\fi
    \end{minipage}
    \end{table}
\else
\begin{table}[t]
    \caption{Performance with learnable Butterfly of different sizes}
    \label{tab:monarch}
    \centering
    \ifarxiv
    \else
    \small
    \fi
    \begin{tabular}{rcc} \toprule
    Block Size & {sCIFAR} & {Speedup} \\ \midrule
    Fixed Butterfly & 91.0 & 1$\times$\\ \midrule
    16 & 91.8 & 1$\times$\\
    32 & 92.4 & 0.9$\times$\\
    256 & \textbf{92.5} & 0.6$\times$ \\
    \bottomrule 
    \end{tabular}
\ifarxiv
\else
\vspace{-2em}
\fi
\end{table}
\begin{table}[t]
    \caption{Performance of replacing MLPs with the long conv extension in a Transformer on WikiText103.}
    \label{tab:wikitext}
    \centering
    \ifarxiv
    \else
    \small
    \fi
    \begin{tabular}{rcc} \toprule
    {Model} & {PPL} & Params  \\ \midrule
    GPT-2-Small & 20.6 & 124M \\
    Monarch-GPT-2-Small & 20.7 & 72M  \\ \midrule
    \fastconv-GPT-2-Small & \textbf{20.4} & 86M  \\
    \bottomrule 
    \end{tabular}
\ifarxiv
\else
\vspace{-2em}
\fi
\end{table}
\fi

Finally, we experimentally evaluate how well the learned Butterfly extension can improve quality on two tasks: sequential CIFAR and WikiText103.

First,
on sequential CIFAR, we use the same architecture as in Section~\ref{subsec:quality}, except with learned Butterfly matrices.
Table~\ref{tab:monarch} shows the results for sequential CIFAR, with varying block sizes.
Block size \num{16} yields lift over the baseline with fixed Butterfly matrices, without sacrificing runtime.
Larger block sizes yield further lift, but at the cost of additional runtime.

Next, on WikiText103, we evaluate the learned Butterfly extension in an alternate setting: replacing MLPs in a Transformer, following~\citep{dao2022monarch}.
In this setting, we leverage the fact that a Butterfly matrix with large block size (\num{256}) approximates a dense matrix multiplication, but has fewer parameters.
We compare our learned Butterfly extension against a Transformer with dense MLPs, and against Transformers where the MLPs have been replaced with Monarch matrices~\citep{dao2022monarch}.
The metric is whether we can achieve the same performance as a Transformer with dense MLPs, but with fewer parameters.

Table~\ref{tab:wikitext} shows the results.
Our extension outperforms both the baseline Transformer and Monarch, outperforming the Transformer with a \num{30\%} reduction in parameters.
This result validates the connection between our learned Butterfly extension and structured sparse matrices.
\fi

\section{Conclusion}
\label{sec:conc}

\ifmefomo
We show that long convolutions are a simple, yet effective approach to long sequence modeling.
We find that regularizing the kernel weights with a squash operator allows long convolutions to achieve strong performance on a variety of long sequence modeling tasks.
Finally, we develop \fastconv to improve the runtime efficiency of long convolutions.
\else
We find that regularizing the kernel weights with a squash operator allows long convolutions to achieve strong performance on a variety of long sequence modeling tasks.
We develop \fastconv to improve the runtime efficiency of long convolutions, using Butterfly decompositions, and we connect convolutions to recent advances in block-sparse matrix multiplication.
\fi

\ifarxiv
\section*{Acknowledgments}

We are very grateful to Sarah Hooper, Arjun Desai, Khaled Saab, Simran Arora, and Laurel Orr for providing feedback on early drafts of this paper and helping to copyedit.
We thank Together Computer for providing portions of the compute used to train models in this paper.
This work was supported in part by high-performance computer time and resources from the DoD High Performance Computing Modernization Program.
We gratefully acknowledge the support of NIH under No. U54EB020405 (Mobilize), NSF under Nos. CCF1763315 (Beyond Sparsity), CCF1563078 (Volume to Velocity), and 1937301 (RTML); US DEVCOM ARL under No. W911NF-21-2-0251 (Interactive Human-AI Teaming); ONR under No. N000141712266 (Unifying Weak Supervision); ONR N00014-20-1-2480: Understanding and Applying Non-Euclidean Geometry in Machine Learning; N000142012275 (NEPTUNE); NXP, Xilinx, LETI-CEA, Intel, IBM, Microsoft, NEC, Toshiba, TSMC, ARM, Hitachi, BASF, Accenture, Ericsson, Qualcomm, Analog Devices, Google Cloud, Salesforce, Total, the HAI-GCP Cloud Credits for Research program,  the Stanford Data Science Initiative (SDSI),
Department of Defense (DoD) through the National Defense Science and
Engineering Graduate Fellowship (NDSEG) Program, 
and members of the Stanford DAWN project: Facebook, Google, and VMWare. The U.S. Government is authorized to reproduce and distribute reprints for Governmental purposes notwithstanding any copyright notation thereon. Any opinions, findings, and conclusions or recommendations expressed in this material are those of the authors and do not necessarily reflect the views, policies, or endorsements, either expressed or implied, of NIH, ONR, or the U.S. Government. 
Atri Rudra's research is supported by NSF grant CCF-1763481.
\fi

\ifmefomo
\bibliographystyle{iclr2023_conference}
\bibliography{main}
\else
\bibliographystyle{icml2022}
\bibliography{main}
\fi
\appendix
\clearpage


\section{Related Work}
\label{sec:related}

\paragraph{State space models}
Following S4~\citep{gu2022efficiently}, deep state space models have been demonstrating promise in sequence modeling.
These models have been especially promising for long sequences, which are challenging for architectures such as Transformers~\citep{vaswani2017attention}, and has required custom approaches to adapt to higher-dimensional data~\citep{dosovitskiy2020image, liang2021evit} or long sequences~\citep{dai2019transformer, wu2022memorizing}.
Deep SSMs have shown state-of-the-art performance on a number of domains, including time series data~\citep{zhang2023effectively,gu2022efficiently,zhou2022deep}, audio~\citep{goel2022s}, visual data~\citep{nguyen2022s4nd}, text~\citep{ma2022mega,mehta2022long,dao2022hungry}, and medical data~\citep{tang2022spatiotemporal}.
A number of methods have also been proposed to simplify the S4 architecture in parameterization~\citep{gupta2022diagonal,gu2022parameterization,smith2022simplified}, make the parameterization more numerically stable~\citep{goel2022s}, or improve the initialization~\citep{gu2022train}.
Some of these approaches have also combined SSMs with other sequence modeling primitives~\citep{hasani2022liquid}, including attention~\citep{ma2022mega,mehta2022long,dao2022hungry}, and~\citet{goel2022s} have used SSMs as a drop-in replacement in audio generation models.
Our work is complementary to these approaches.
For example, one way to leverage principled initialization techniques is to apply them directly to the long convolutions.
Our work also suggests that long convolutions may be promising architectures for the downstream applications where SSMs have shown strong performance.

\paragraph{Convolutions}
Convolutions have a long history in signal processing~\citep{smith1997scientist} and machine learning, especially in computer vision~\citep{he2016deep,krizhevsky2017imagenet,lecun1998gradient}.
Most models are based on short, localized convolutions~\citep{trockman2022patches}, and most libraries are optimized for short convolutions~\citep{paszke2019pytorch}.
Recently, in conjunction with the development of state space models, there has been growing interesting in models that use long convolutions~\citep{romero2021flexconv, romero2021ckconv, varol2017long}, often with implicit representations~\citep{li2022makes, guibas2021adaptive, lee2021fnet}.
Approaches such as CKConv~\citep{romero2021ckconv} and SGConv~\citep{li2022makes} have shown that convolutions can be effective for sequence modeling, but require parameter counts that grow sub-linearly in sequence length and build in explicit biases into the parameterization~\citep{li2022makes}.
Our work provides additional support for the use of long convolutions for sequence modeling, and suggest that---with the right regularization---long convolutions can be used successfully for sequence modeling without controlling for parameter counts.

\paragraph{FFT Algorithms}
The computational feasibility of long convolution models depends on the Fast Fourier Transform (FFT).
The Cooley-Tukey FFT algorithm, published in 1965 \citep{cooley1965an}, enabled convolution and Fourier transforms to scale in the length dimension from $O(N \log N)$ instead of $O(N^2)$.
Subsequently, many alternative algorithms for efficiently computing the Fourier transform have emerged, 
including algorithms for computing the FFT in parallel \citep{ayinala2011pipelined}.
These algorithms have enabled fundamental progress in a range of disciplines, including control theory~\citep{brigham1988fast,bekele2016cooley} and signal processing~\citep{oppenheim1978applications,oppenheim2001discrete}.  
A survey of methods is included in \citeauthor{chu1999inside, bahn2009parallel}.

As FFTs prove more useful for modern deep learning workloads---e.g., through long convolutions---new techniques are required to run them efficiently.
Of particular interest to our work is making FFTs run efficiently on GPUs with specialized matrix multiplication units, such as tensor cores.
For example, an A100 GPU has a maximum of 312 TFLOPs/s of FP16 with
tensor cores, but only 20 TFLOPs/s of FP32 (and 40 TFLOPs/s of FP16) without
tensor cores~\citep{nvidia2020nvidia}.
This trend started with the V100 GPUs~\citep{nvidia2017nvidia} and has continued with the H100 GPUs~\citep{nvidia2022nvidia}.
Our work is related to and draws from efforts in the high-performance computing community to accelerate FFTs given these new hardware primitives~\citep{li2021tcfft}, but focuses specifically on using them in convolutions.
In the convolution workload, it is important to mitigate IO costs and increase FLOP utilization in concert.

\paragraph{Sparse Structured Matrices}
Sparse structured matrices have recently received a great deal of attention as a promising research topic in making machine learning models more runtime- and parameter-efficient.
Sparse training has a long history in machine learning, including work in pruning neural networks~\citep{han2015deep, han2015learning, sanh2020movement, lin2017runtime, dong2017learning} and finding lottery tickets~\citep{frankle2018lottery, frankle2019stabilizing, frankle2020linear}.
Structured matrices are another approach to making models more efficient.
Structured matrices have subquadratic ($o(n^2)$ for dimension $n \times n$) parameters and runtime, such as sparse and low-rank matrices, and fast transforms (Fourier, Chebyshev, sine/cosine, orthogonal polynomials)~\citep{dao2022monarch}.
Critically, simple divide-and-conquer schemes can lead to fast algorithms for many structured matrices~\citep{de2018two}, and structured matrices can be used to represent many commonly used fast transforms~\citep{sindhwani2015structured, kailath1979displacement, eidelman1999new}.
Our connection to these matrices comes through butterfly matrices~\citep{dao2019learning,chen2021pixelated,parker1995random}, which have been shown to be expressive and hardware-efficient~\citep{dao2022monarch}.
Butterfly matrices have also been used in kernel methods~\citep{choromanski2019unifying, munkhoeva2018quadrature} and deep learning methods~\citep{prabhu2020butterfly, lin2021deformable, ailon2021sparse}, which may suggest other fruitful avenues of future work for long convolutions with a learned Butterfly formulation.
Our work suggests that using long convolution models may offer an inroads to utilizing structured matrices in deep learning.

\section{Additional Experiments\label{sec:exp_supp}}
\ifmefomo
    \caption{Image classification on flattened images.}
    \label{tab:scifar}
    \centering
    \small
    \begin{tabular}{rc} \toprule
    {Model} & {sCIFAR} \\ \midrule
    Transformer & 62.2 \\ \midrule
    LSTM & 63.0 \\
    r-LSTM & 72.2 \\
    UR-LSTM & 71.0 \\
    UR-GRU & 74.4 \\
    HIPPO-RNN & 61.1 \\
    LipschitzRNN & 64.2 \\ 
    CKConv & 64.2 \\ \midrule
    S4-LegS & \underline{91.8} \\
    S4-FouT & 91.2 \\
    S4D-LegS & 89.9 \\
    S4D-Inv & 90.7 \\
    S4D-Lin & 90.4 \\
    \midrule
    Long Conv, Random  & 91.0 \\
    Long Conv, Geom Init  & \textbf{92.1} \\
    \bottomrule 
    \end{tabular}
\ifarxiv
\else
\vspace{-1em}
\fi
    \caption{Image classification on 2D images.}
    \label{tab:s42d}
    \centering
    \small
    \begin{tabular}{rc} \toprule
    {Model} & {CIFAR} \\ \midrule
    S4ND-ISO & \textbf{89.9} \\ \midrule
    Long Conv 2D-ISO, Rand init  & 88.1 \\
    Long Conv 2D-ISO, Geom init  & \underline{89.1} \\
    \bottomrule 
    \end{tabular}
    \ifarxiv
    \else
    \vspace{-1em}
    \fi
    \caption{Evaluation on brain fMRI data.}
    \label{tab:fmri}
    \centering
    \small
    \begin{tabular}{rc} \toprule
    {Model} & {MAE} \\
    \midrule
    {Transformer} & {0.68} \\ 
    {H3} & {0.70} \\
    {H3 + Long Convs, Rand Init} & \underline{0.58} \\
    {H3 + Long Convs, Geom Init} & \bf{0.54} \\
    \bottomrule 
    \end{tabular}
\ifarxiv
\else
\vspace{-1em}
\fi
\subsection{Image Classification}
We evaluate long convolutions on image classification.
We evaluate two settings which have been used to evaluate SSMs and sequence models: 1D pixel-by-pixel image classification, and 2D image classification.
These settings are challenging for sequence modeling, as they require modeling complex spatial relationships between image pixels in a continuous space.
For the 1D case, we use long convolutions as a drop-in replacement for the SSM layer in the state-of-the-art S4 architecture.
For the 2D case, we replace the S4 layers in S4ND~\citep{nguyen2022s4nd} with 2D long convolution filters.

Tables~\ref{tab:scifar} and~\ref{tab:s42d} show the results.
On 1D image classification, long convolutions again match the performance of S4, even with random initializations, while their performance improves further by \num{1.3} points when using the exponential decay initialization.
On 2D image classification, long convolutions come within \num{0.8} points of the state-of-the-art S4ND model---which suggests that higher dimensions may require different techniques or inductive bias to recover the same performance.
\else
\fi

\subsection{Regularization in Frequency Domain}

\begin{table*}[t]
    \caption{Validation accuracy of different models on the LRA benchmark. This table includes extra experiments where we run the \smooth operator over the frequency domain. Best in bold, second best underlined.}
    \label{tab:lra_freq}
    \scriptsize
    \centering
    \begin{tabular}{rccccccc} \toprule
    {Model} & {ListOps} & {Text} & {Retrieval} & {Image} & {Pathfinder} & {Path-X} & {Avg} \\ \midrule
    Transformer  & 36.4  & 64.3 & 57.5  & 42.4 & 71.4 & \xmark  & 53.7 \\
    Nyströmformer & 37.2   & 65.5 & 79.6   & 41.6 & 70.9 & \xmark  & 57.5  \\
    Reformer & 37.3 & 56.1 & 53.4 & 38.1 & 68.5 & \xmark & 50.6 \\
    BigBird & 36.1 & 64.0 & 59.3 & 40.8 & 74.9 & \xmark & 54.2 \\
    Linear Trans. & 16.1 & 65.9 & 53.1 & 42.3 & 75.3 & \xmark & 50.5 \\
    Performer & 18.0 & 65.4 & 53.8 & 42.8 & 77.1 & \xmark & 51.2 \\
    \midrule
    S4-LegS  & 59.6  & 86.8 & 90.9  & 88.7 & \underline{94.2} & \underline{96.4}  & \underline{86.1}  \\ 
    S4-FouT  & 57.9  & 86.2 & 89.7  & \textbf{89.1} & \textbf{94.5} & \xmark  & 77.9  \\
    S4-LegS/FouT & \underline{60.5} & 86.8 & 90.3 & \underline{89.0} & 94.4 & \xmark & 78.5 \\
    S4D-LegS & \underline{60.5} & 86.2 & 89.5 & 88.2 & 93.1 & 92.0 & 84.9 \\ 
    S4D-Inv & 60.2 & \underline{87.3} & \underline{91.1} & 87.8 & 93.8 & 92.8 & 85.5 \\
    S4D-Lin & 60.5 & 87.0 & 91.0 & 87.9 & 94.0 & \xmark & 78.4 \\
    S4 (Original) & 58.4 & 76.0 & 87.1 & 87.3 & 86.1 & 88.1 & 80.5 \\ \midrule
    Long Conv, Rand & 53.4 & 64.4 & 83.0 & 81.4 & 85.0 & \xmark & 69.5 \\
    Long Conv, Rand + \smooth & 59.8 & 68.7 & 86.6 & 79.3 & 86.1 & \xmark & 71.8 \\
    Long Conv, Rand + \smooth, Freq & 56.1 & 67.9 & 86.8 & 85.2 & 88.3 & \xmark & 72.4 \\
    Long Conv, Rand + \squash & 60.3 & 87.1 & 90.0 & 88.3 & 94.0 & \textbf{96.9} & \underline{86.1} \\
    Long Conv, Rand + \squash + \smooth & 59.7 & 72.8 & 88.6 & 80.8 & 90.1 & \xmark & 73.7 \\
    Long Conv, Rand + \squash + \smooth, Freq & 59.7 & 72.8 & 88.6 & 85.7 & 88.3 & 84.9 & 78.8 \\
    Long Conv, Exp + \squash & \textbf{62.2} & \textbf{89.6} & \textbf{91.3} & 87.0 & 93.2 & 96.0 & \textbf{86.6} \\
    \bottomrule 
    \end{tabular}
\end{table*}
In earlier experiments, we also experimented with smoothing convolutions directly in the frequency domain.
We convert the convolution kernel $\vK$ to frequency domain with an FFT, apply the smoothing operator \smooth, and then convert it back to the time domain with an inverse FFT.
Table~\ref{tab:lra_freq} shows the results on LRA, with the convolutions smoothed in frequency domain denoted by ``\smooth, Freq.''
The performance is similar to smoothing in time domain, but underperforms using the \squash operator on its own, so we elected to just use the \squash operator in the remaining experiments.

\subsection{Time Series Forecasting}
\begin{table*}[t]
    \centering
    \tiny
    \caption{Univariate long sequence time-series forecasting results on ETTh1 Informer benchmark. Comparisons across five horizon prediction settings. Best mean squared error (MSE) and mean absolute error (MAE) in bold. Numbers reported from \cite{gu2022efficiently}. Long Convs outperforms S4 and obtains best MSE and MAE in four out of five evaluation settings.}
    \label{tab:informer-s-long}
    \begin{tabular}{@{}cccccccccc@{}}
    \toprule
    Methods & Long Convs & \multicolumn{1}{c}{S4}          & \multicolumn{1}{c}{Informer}  & \multicolumn{1}{c}{LogTrans} & \multicolumn{1}{c}{Reformer} & \multicolumn{1}{c}{LSTMa} & \multicolumn{1}{c}{DeepAR} & \multicolumn{1}{c}{ARIMA} & \multicolumn{1}{c}{Prophet} \\ \midrule
    Metric  & MSE~MAE              & MSE~MAE            & MSE~MAE          & MSE~MAE          & MSE~MAE          & MSE~MAE         & MSE~MAE         & MSE~MAE         & MSE~MAE          \\ \midrule
    24  & \textbf{0.06}~0.20            & 0.06~\textbf{0.19} & 0.10~0.25   & 0.10~0.26        & 0.22~0.39        & 0.11~0.27       & 0.11~0.28       & 0.11~0.28       & 0.12~0.28        \\
    48  & \textbf{0.07}~\textbf{0.21}   & 0.08~0.22          & 0.16~0.32   & 0.17~0.33        & 0.28~0.45        & 0.19~0.36       & 0.16~0.33       & 0.18~0.42       & 0.17~0.33         \\
    168 & \textbf{0.07}~\textbf{0.21}   & 0.10~0.26          & 0.18~0.35   & 0.21~0.38        & 1.52~1.19        & 0.24~0.39       & 0.24~0.42       & 0.40~0.50       & 1.22~0.76        \\
    336 & 0.08~\textbf{0.23}            & \textbf{0.08}~0.23 & 0.22~0.39   & 0.23~0.40        & 1.86~1.12        & 0.59~0.70       & 0.45~0.55       & 0.47~0.59       & 1.55~1.82         \\
    720 & \textbf{0.09}~\textbf{0.24}   & 0.12~0.27          & 0.27~0.44   & 0.27~0.46        & 2.11~1.44        & 0.68~0.77       & 0.66~0.71       & 0.66~0.77       & 2.74~3.25        \\ \bottomrule
    \end{tabular}
\end{table*}
Time series forecasting is another challenging modality for sequence modeling, which requires reasoning over multiple time contexts.
We evaluate the performance of long convolutions on different future horizon prediction windows in ETTh$_1$, a real-world long sequence time series forecasting task from the Informer benchmark \cite{zhou2021informer}. Following the original S4 paper, we evaluate on the univariate ETTh$_1$ task, which involves predicting electricity transformer temperature at hour-long granularities (i.e., 24, 48, 168, 336, and 720 hours in the future). For each prediction task, we use the same number of hours before as a look-back window to input to the model. As LongConvs can be a drop-in replacement for the S4 kernel, we also follow the approach taken in S4 that simply masks out the future time steps in the input sequence and treat the task as a masked sequence-to-sequence transformation.
Table~\ref{tab:informer-s-long} shows the results.
Long convolutions match or outperform S4 on all context windows, and outperforms custom hand-crafted architectures designed specifically for time series forecasting.

\subsection{Brain fMRI Downstream Adaptation}
\begin{table}[t]
    \caption{Downstream performance on brain fMRI data.}
    \label{tab:fmri_downstream}
    \centering
    \begin{tabular}{lrc} \toprule
    {Dataset} & {Model} & {F1} \\
    \midrule
    {MDTB} & {Transformer} & {91.8} \\
    {} & {H3} & {92.0} \\ 
    {} & {H3 + Long Convs, Rand Init} & \bf{92.1} \\
    {} & {H3 + Long Convs, Exp Init} & {91.6} \\
    \midrule
    {HCP} & {Transformer} & {83.4} \\
    {} & {H3} & {82.6} \\ 
    {} & {H3 + Long Convs, Rand Init} & {82.3} \\
    {} & {H3 + Long Convs, Exp Init} & \bf{83.6} \\
    \bottomrule 
    \end{tabular}
\end{table}

We further evaluate the performance of the pre-trained models in two benchmark mental state decoding datasets from the Human Connectome Project~\citep[HCP;][]{barch2013function} and multi-domain task battery \citep[MDTB;][]{king2019functional}, spanning $20$ and $26$ distinct mental states respectively.
To adapt the pre-trained models to the mental state decoding (i.e., classification) task, we add a learnable classification embedding $E^{cls} \in \mathbb{R}^{n}$ to the end of input sequences $X$ and forward the model's corresponding prediction to a decoding head $p(\cdot)$, composed of a dense hidden layer with $e$ model units (one for each embedding dimension, with $tanh$ activation) as well as a $softmax$ output layer (with one model unit $i$ for each considered mental state in the data).
Accordingly, we adapt models by optimizing a standard cross entropy loss objective: $- \sum_i y_i \log \ {p(f(E^X))_i}$, where $y_i$ indicates a binary variable that is $1$ if $i$ is the correct mental state and $0$ otherwise. 
We always begin downstream adaptation with the pre-trained model parameters and  allow all parameters to change freely during training.
We randomly split each of the two downstream datasets into distinct training ($90\%$ of fMRI runs) and test ($10\%$ of fMRI runs) datasets and adapt models for $1,000$ training steps at a mini-batch size of $256$ and a learning rate of $5e^{-5}$ (otherwise using the same learning parameters as for upstream training).
During training, we sample sequences from the fMRI datasets according to the accompanying event files, which specify the beginning and end of each experimental trial underlying a mental state~\citep[when accounting for the temporal delay of the haemodynamic response function; for details, see ][]{thomas2022self}.

The adapted H3 variants with long convolutions perform on par with the other models in accurately identifying the mental states of the downstream evaluation datasets (see Table \ref{tab:fmri_downstream}: F1-scores are macro-averaged).

\section{Methods Details\label{sec:methods_details_supp}}
We discuss details of our methods.
\ifmefomo
\subsection{Kernel Fusion}
Naive implementations of the FFT convolution incur expensive GPU memory IO.
Each FFT and inverse FFT operation requires at least one read and write of the input sequence from GPU memory, and so does the pointwise multiplication operation.
For long sequences, the IO costs may be even worse: the entire input sequence cannot fit into SRAM, so optimized implementations such as cuFFT~\citep{cufft} must take multiple passes over the input sequence using the Cooley-Tukey decomposition of the FFT~\citep{cooley1965an}.
Following \textsc{FlashAttention}~\citep{dao2022flashattention}, \fastconv's first contribution is to fuse the entire FFT convolution into a single kernel and compute the result directly in GPU SRAM to avoid this overhead.

\begin{table}[t]
    \caption{Runtime, GLOPs, and FLOP util for the Butterfly decomposition with different block sizes $r$ for sequence length 4096, on A100 with batch size $128$, head dimension $32$.}
    \label{tab:block_micro}
    \ifarxiv
    \else
    \small
    \fi
    \centering
    \begin{tabular}{rccc} \toprule
    {Block Size} & {Runtime (ms)} & GLOPs & FLOP Util \\ \midrule
    2 & 0.52 & 2.0 & 1.3\% \\
    16 & 0.43 & 8.1 & 6.0\% \\
    64 & 0.53 & 21.5 & 13.0\% \\
    256 & 0.68 & 64.5 & 30.4\% \\
    \bottomrule 
    \end{tabular}
\end{table}

\subsection{Butterfly Decomposition}
Kernel fusion reduces the IO requirements, but the fused FFT operations still cannot take full advantage of specialized matrix multiply units on modern GPUs, such as Tensor Cores on Nvidia GPUs, which perform fast 16 $\times$ 16 matrix multiplication.
We appeal to a classical result, also known as the four-step or six-step FFT algorithm~\citep{bailey1990ffts}, that rewrites the FFT as a series of block-diagonal Butterfly matrices~\citep{parker1995random} interleaved with permutation.

The Butterfly decomposition states that we can decompose an $N$-point FFT into a series of FFTs of sizes $N_1$ and $N_2$, where $N = N_1N_2$.
Conceptually, the algorithm reshapes the input as an $N_1 \times N_2$ matrix, applies $N_1$ FFTs of size $N_2$ to the columns, multiplies each element by a twiddle factor, and then applies $N_2$ FFTs of size $N_1$ to the rows.

More precisely, let $\vF_N$ denote the DFT matrix corresponding to taking the $N$-point FFT.
Then, there exist permutation matrices $\vP$, and a diagonal matrix $\vD$, such that $\vF_N = \vP(\vI_{N_2} \otimes \vF_{N_1})\vP^T \vD(\vI_{N_1} \otimes \vF_{N_2})\vP$.
$\vP$ denotes a permutation matrix that reshapes the input to $N_1 \times N_2$ and takes the transpose, $\vD$ denotes a diagonal matrix with the twiddle factors along the diagonal, $\otimes$ denotes the Kronecker product, and $vI_{N_i}$ and $\vF_{N_i}$ are the identity and DFT matrices of size $N_i \times N_i$.
Precise values for $\vF_{N_i}$, $\vD$, and $\vP$ are given in Appendix~\ref{sec:methods_details_supp}.

The Butterfly decomposition incurs $O(Nr \log N / \log r)$ FLOPS for a sequence length $N = r^p$, with \textit{block size} $r$.
In general FFT implementations, $N$ is typically padded to a power of two, so that the block size can be set to $2$ to minimize the total number of FLOPS.
However, on GPUs with a specialized $b \times b$ matrix multiply unit, the FLOP cost of computing an $r \times r$ matrix multiply with $r < b$ is equivalent to performing a single $b \times b$ matrix multiply.
Thus the actual FLOP count scales as $O(N b \log N / \log r)$ for $r < b$.
Increasing the block size up to $b$ actually \textit{reduces} the FLOP cost.

Table~\ref{tab:block_micro} demonstrates this tradeoff on an A100 GPU, which has specialized matrix multiply units up to 16 $\times$ 32.
Runtime decreases as $r$ increases from 2, even though theoretical FLOPS increase.
Once $r > b$, runtime begins increasing as actual FLOPS increase as well.
\else
\fi
\ifmefomo
\else
\subsection{Butterfly Decompositions}
\fi
We describe how to construct $\vD$ in the Butterfly decomposition, and $\vB$ in the three pass algorithm.

\paragraph{Twiddle Matrices}
We describe how to construct $N_1 \times N_2$ Twiddle matrices.

Let $M \in \mathbb{C}^{N_1 \times N_2}$.
Then $M_{j,k} = \exp(-2\pi i jk / N)$.
The twiddle factors $\vD$ can be constructed by flattening $M$ and using them along the diagonal of $\vD$.

\paragraph{Butterfly Matrix}
We construct $\vB$ in the three pass algorithm.

Let $\vB^{(m)}$ denote the Butterfly matrix that needs to be constructed for a three pass algorithm with $N = lm$, and assume that $m$ is a power of 2.
$\vB^{(m)}$ is a block matrix, where each block is a diagonal matrix.
In particular, we have:
$$
\vB = \begin{bmatrix}
    \vD_{1,1} & \dots & \vD_{1, m} \\
    \vdots & \ddots & \vdots \\
    \vD_{m, 1} & \dots & \vD_{m, m},
\end{bmatrix}.
$$

We show how to construct $\vD_{j, k}$.
$\vD_{j, k}$ is a diagonal matrix of size $l \times l$.
The entries of $\vD_{j, k}$ are given by the following:
$$
\vD_{j, k}[\tau] = \exp(-2i\pi k (jl + \tau) / N).
$$

\subsection{Additional details about the three pass algorithm}

We share a few additional details about the three pass algorithm that allow for efficient training.

The butterfly matrices $\vB$ have complex coefficients.
Typically, we train models over real time series.
This mismatch has the potential to increase the amount of GPU memory IO: it is necessary to read $N$ real numbers, but write $N$ complex numbers.

We can alleviate this problem by using a well-known transformation between a real FFT of length $2L$ and a complex FFT of length $L$~\citep{brigham1988fast}.
In essense, a real FFT of length $2L$ can be converted into a complex FFT of length $L$.
In our algorithm, we exploit this as follows:
\begin{itemize}
    \item Given an input of real points $N$, reshape the input to be a complex input of length $N / 2$.
    \item Compute the complex FFT convolution over the input of length $N / 2$ using the three pass algorithm.
    \item Convert the output to be a real output of length $N$.
\end{itemize}
The first and last steps can be fused with a Butterfly matrix multiplication kernel, thereby keeping the total IO cost the same as the original algorithm.

\section{Theory\label{sec:theory_supp}}

\subsection{Three-Pass Algorithm}
\ifmefomo

The full algorithm for \fastconv for $N > l$ is shown in Algorithm~\ref{alg:fastconv}.
\begin{algorithm}[t]
    \algsetup{linenosize=\tiny}
    \caption{\label{alg:fastconv} \fastconv}
    \small
    \begin{algorithmic}[1]
      \REQUIRE Input $u \in \mathbb{R}^{B \times H \times N}$, $\vK \in \mathbb{R}^{H \times N}$, $\vD \in \mathbb{R}^{H}$, where $N = lm$ is the sequence length, $H$ is the head dimension, and $B$ is the batch size. 
      \STATE $\hat{\vK} \gets FFT(\vK)$
      \STATE $\vD'_{\vK} \gets \vP (\hat{\vK}) \vP^{-1}$ 
      \STATE $u \gets \overline{\vB}^{-1}u$
      \STATE Compute $u \gets (\vI_{m} \otimes \overline{\vF}_{l}) \vD'_{\vK} (\vI_{m} \otimes \vF_{l}) u$
      in parallel across $m$ streaming multiprocessors
      \STATE Return $\overline{\vB} u \in \bR^{B \times H \times N}$
    \end{algorithmic}
  \end{algorithm}

We show that Algorithm~\ref{alg:fastconv} is correct, and that it can be computed in three passes over the input sequence.
\begin{proposition}
    \label{prop:fastconv}
    Algorithm~\ref{alg:fastconv} computes the convolution $u \ast \vK$ with at most three passes over the input sequence $u$.
\end{proposition}
\else
\fi

We prove Proposition~\ref{prop:fastconv}.

    \paragraph{Convolution}
    Recall that a convolution between two vectors $u$ and $k$ of length $N$ is given by the following:
    $$
    u \ast k = \bar{\vF}_L \text{Diag}(\vF_L k) \vF_L u.
    $$
    We can precompute $\bar{\vF}_L k$, since it is shared across all inputs in a batch.
    Let $\vD = \bar{\vF}_L k$.
    Then, the above is given by:
    $$
    u \ast k = \bar{\vF}_L \vD \vF_L u.
    $$
    
    \paragraph{Decomposition}
    One property of $\vF_L$ is that it can be decomposed.
    For example, if $L = 2l$, then we can write the following:
    $$
    \vF_{2l} = \vB \begin{bmatrix}
    \vF_{l} & 0 \\
    0 & \vF_{l}
    \end{bmatrix} \vP,
    $$
    where $\vP$ is a permutation matrix (in this case, an even-odd permutation), and $\vB$ is a Butterfly matrix.
    
    We can leverage this to re-write a convolution of length $2l$.
    Let $u$ and $k$ be vectors of length $2l$.
    Then, we can write the following:
    \begin{align*}
    u \ast k &= \bar{\vF}_{2l} \vD \vF_{2l} u \\
    &= \bar{\vF}_{2l} \vD \bar{\vF}_{2l}^{-1} u \\
    &= \bar{\vB} \begin{bmatrix}
        \bar{\vF}_{l} & 0 \\
        0 & \bar{\vF}_{l}
    \end{bmatrix} \vP \vD \vP^{-1} \begin{bmatrix}
        \bar{\vF}_l^{-1} & 0 \\
        0 & \bar{\vF}_l^{-1} \\ 
    \end{bmatrix} \bar{\vB}^{-1} u \\
    &= \bar{\vB} \begin{bmatrix}
        \bar{\vF}_{l} & 0 \\
        0 & \bar{\vF}_{l}
    \end{bmatrix} \vD' \begin{bmatrix}
        \bar{\vF}_l^{-1} & 0 \\
        0 & \bar{\vF}_l^{-1} \\ 
    \end{bmatrix} \bar{\vB}^{-1} u,
    \end{align*}
    for some diagonal matrix $\vD'$.
    Note that the three terms in the middle can be computed in parallel.
    
    This pattern extends to $L = 2^m l$, and yields $2^m$ parallelism in the product.

    It remains to show that each of the Butterfly matrices can be computed with a single read/write over the input sequence.

    Recall that the Butterfly matrices have the following form:
    $$
    \vB = \begin{bmatrix}
        \vD_{1,1} & \dots & \vD_{1, m} \\
        \vdots & \ddots & \vdots \\
        \vD_{m, 1} & \dots & \vD_{m, m},
    \end{bmatrix}
    $$
    where the $\vD_{i, j}$ are diagonal matrices of size $l \times l$.

    A matrix-vector multiply $y = \vB u$ can be partitioned on a GPU as follows.
    Suppose that each SM has enough shared memory to store $l$ elements of the input.
    Let there be $m$ SMs processing this input.
    Each SM will read $l$ input and write $l$ output, for $ml = N$ total reads and writes.

    Specifically, SM i will read
    \begin{align*}
    &u[(l/m)i:(l/m)(i+1)], \\
    &u[l+(l/m)i:l+(l/m)(i+1)], \dots,\\
    &u[(m-1)l+(l/m)i:(m-1)l+(l/m)(i+1)].
    \end{align*}
    These inputs are exactly the inputs needed to compute:
    \begin{align*}
        &y[(l/m)i:(l/m)(i+1)], \\
        &y[l+(l/m)i:l+(l/m)(i+1)], \dots,\\
        &y[(m-1)l+(l/m)i:(m-1)l+(l/m)(i+1)].
    \end{align*}
    The SM can then distribute these portions of the matrix-vector multiply to the independent threads of the SM.

    This completes the proof.

\subsection{Expressivity of Long Convolutions}
\label{sec:layer_analysis}
We show that long convolutions and SSMs are equivalent in expressivity (the subset relation in Figure~\ref{fig:banner} right is actually set equality).

\begin{proposition}
    \label{prop:expressive}
    Let $M$ be a positive integer that evenly divides $N$.
    Any convolution kernel of length $N$ can be written as the sum of $N / M$ diagonal SSMs with hidden state $M$.
\end{proposition}
\begin{proof}
    For the case $M=1$, consider a diagonal SSM with $\vA \in \bR^{N \times N}$ diagonal with entries $a_1, \dots, a_{N}$, and $\vB \in \bR^{N \times 1}$.
    For simplicity, we will roll $\vC$ into $\vB$ and set $\vD= 0$.
    
    This SSM gives rise to the following kernel $\vK$ with entries:
    $$
        \vK_i = A^{i} \vB = \sum_{j=1}^N a_j^i b_j.
    $$
    This is equivalent to
    $$
        \vK = \vV \vB, 
    $$
    where $\vV$ is the transpose of a Vandermonde matrix
    $$\vV = \begin{bmatrix}
    1 & a_{1} & a_{1}^{2} & \dots & a_{1}^{N-1} \\ 
    1 & a_{2} & a_{2}^{2} & \dots & a_{2}^{N-1} \\
    \hdotsfor{5} \\
    1 & a_{N} & a_{N}^{2} & \dots & a_{N}^{N-1}
    \end{bmatrix}^T.
    $$
    Vandermonde matrices have a determinant that is nonzero 
    if and only if $a_1,...,a_{N}$ are all distinct. 
    Thus $\vV^T$ is invertible if $a_1,...,a_{N}$ are distinct and hence $\vV$
    is also invertible if $a_1,...,a_{N}$ are distinct.
    Given a kernel $\hat{\vK}$, we can thus express that kernel by picking any $a_1,...,a_{N}$ that are distinct and then picking 
    $\vB = \vV^{-1} \hat{\vK}$, then $\vV \vB = \vV\vV^{-1} \hat{\vK} =\hat{\vK}$, finishing the proof.
    
    In the case where $M>1$ we have consider a diagonal SSM with $\vA \in \bR^{N \times N}$ diagonal with entries $a_1, \dots, a_{N}$, and $\vB \in \bR^{N \times 1}$.
    Partition, the state $N$ into $N / M$ partitions of size $M$.
    Let $\sigma(i,j)$ denote the partition function that bijectively maps $(i, j)$ pairs to $[1, \dots, N]$ for $1 \leq i \leq N / M, 1 \leq j \leq M$.
    
    Then the convolution kernel has the following entries $\vK_i$:
    $$
        \vK_l = \sum_{i=1}^N a_i^l b_i = \sum_{i = 1}^{N / M} \sum_{j=1}^M a_{\sigma(i,j)}^l b_{\sigma(i,j)}.
    $$
    Consider the inner sum $\sum_{j=1}^M a_{\sigma(i,j)}^l b_{\sigma(i,j)}$.
    This defines a convolution kernel given by a diagonal SSM with hidden state $M$, $\vA$ with diagonal entries $[a_{\sigma(i,1)}, \dots, a_{\sigma(i, M)}]$, and $\vB = [b_{\sigma(i,1)}, \dots, b_{\sigma(i, M)}]^T$.
    
    Thus, this diagonal SSM with hidden state $N$ is the sum of $N / M$ diagonal SSMs with hidden state $M$.
    \end{proof}

Proposition~\ref{prop:expressive} suggests that long convolutions and SSMs have fundamentally the same expressive power, especially when SSMs are used in a deep architecture that stacks multiple independent SSMs in layers.
The significance of this result is that this allows us to view SSMs and general long convolutions
as the same construct.
\ifmefomo
\else
\section{Additional Methods\label{sec:methods_additional}}

We discuss some additional methods that we tried but did not include in the main body of the paper.

\subsection{Constant-Recursive Kernels}
In early explorations, we explored a constant-recursive formulation of convolution kernels as a mechanism for hardware efficiency.
We ultimately did not go with this route due to the development of \fastconv.

We wish to develop kernels such that the output of convolving the kernels with a signal can be computed recurrently.
The goal is to develop long convolution kernels that can be computed efficiently.

We show that constant-recursive kernels satisfy our requirements.
In particular, the output of convolving a constant-recursive kernel with a signal results in an output that is also constant-recursive.
We show that our formulation of constant-recursive kernels is expressive enough to capture S4D kernels---and by corollary by Proposition~\ref{prop:expressive}, any kernel.

\subsubsection{Background: Constant-Recursive Sequence}

We define constant-recursive sequences.

\paragraph{Constant-Recursive Sequence}
A constant-recursive sequence is a sequence of numbers $s_1, s_2, \dots$ that satisfies the following recursive function:
$$
s_n = a_1s_{n-1} + a_2s_{n-2} + \dots + a_p s_{n-p} = \Sigma_j^p a_j s_{n-j},
$$
for all $n > p$, where $a_1, \dots, a_p$ are constants.
We will call $p$ the \textbf{power} of the constant-recursive sequence (this terminology may not be standard).

\subsubsection{Constant-Recursive Kernel}
We use the idea of constant-recursive sequences to define a constant-recursive kernel.
Our key insight is that the convolution of a constant-recursive sequence with a signal is itself a constant-recursive sequence.
We will define our kernel as the sum of $d$ constant-recursive kernels, where $d$ is a hidden state dimension (equivalent to the state dimension $d$ in a state space model).
We will show that our formulation is expressive enough to capture S4D.

We will define the convolution kernel $\bar{K}$ through a recurrence relation with dimension $d$, and power $p$: 
\begin{align*}
    \bar{K}_{i} &= \Sigma_{r=1}^{d} K_{i,r} \\
    K_{i,r} &= \begin{cases}
k_{i,r} & \text{for } 1 \leq i \leq p \\
\Sigma_{j=1}^{\text{min}(p, i-p)} a_{j,r} K_{i-p-j+1,r} & \text{otherwise}
\end{cases}
\end{align*}
Note that each $K_{i, r}$ is a form of constant-recursive sequence, with a form of delay in the sequence.
In particular, we have that $K_{i,r}$ depends on $K_{i-2p+1,r}, \dots, K_{i-p,r}$.
We make this choice for computational reasons---it means we can compute $K_{p+1,r}, \dots, K_{2p,r}$ at once, then $K_{2p+1,r}, \dots, K_{3p,r}$, etc each in one go.
Formally, it is equivalent to a constant-recursive sequence with twice the power, but where the first $p$ constants are all zeros.

The special cases $d=1, p>1$ and $d>1, p=1$ are worth analyzing separately to develop some intuition about what this convolution does.

\subsubsection{Case $d=1, p>1$:}
We first analyze the case when $d=1$ to develop some intuition about what this kernel is expressing.
We will see that using this kernel in a convolution yields a constant-recursive output.

When $d=1$, the kernel expression becomes
\begin{align} 
    \bar{K}_{i,1} &= \begin{cases}
k_{i,1} & \text{for } 1 \leq i \leq p \\
\Sigma_{j=1}^{\text{min}(p, i-p)} a_{j,1} \bar{K}_{i-p-j+1,1} & \text{otherwise}
\label{eq:ssconv}
\end{cases}
\end{align}

We show that using this kernel in a convolution results in a constant-recursive output sequence:
\begin{proposition}
    \label{prop:recurrent_k}
    Let $\bar{K} \in \mathbb{R}^L$ be a kernel defined by Equation~\ref{eq:ssconv}, and let $u \in \bR^L$.
    Then $y = u \ast \bar{K} \in \bR^L$ is given by the following:
    \begin{equation}
        \label{eq:recurrent_y}
        y_i = \Sigma_{j=1}^{\text{min}(i, p)} k_j u_{i-j+1} + \Sigma_{j=1}^{\text{min}(p, i-p)} a_j y_{i-d-j+1}
    \end{equation}
\end{proposition}
The significance of Proposition~\ref{prop:recurrent_k} is that $y_i$ has the exact same constant-recursive structure as $K_{i, r}$ -- and can thus be computed as a recurrently.

\paragraph{Equivalent SSM}
We construct the $\vA$, $\vB$, $\vC$ matrices for an SSM that produces this kernel.
Let $\mathbf{B} = [k_{1,1}, \dots, k_{1,p}]^T$, $\mathbf{C} = [1, 0, \dots, 0]^T$, and $\mathbf{A} \in \mathcal{R}^{p \times p}$ be the following (inverted companion) matrix:
$$
\mathbf{A} = 
\begin{bmatrix}
0 & 1 & 0 & \dots & 0 \\
0 & 0 & 1 & \dots & 0 \\
\vdots \\
0 & 0 & 0 & \dots & 1 \\
a_{1,1} & a_{1,2} & a_{1,3} & \dots & a_{1,p}
\end{bmatrix}
$$
then $\bar{K}_i = \mathbf{C}^T \mathbf{A}^{i-1} \mathbf{B}$, which reveals that this constant-recursive matrix is equivalent to an SSM.

\subsubsection{Case $d>1, p=1$:}
This case, where the constant-recursive sequence as power $p=1$, recovers S4D with $d>1$.

The kernel definition simplifies to:
\begin{align*}
    \bar{K}_{i} &= \Sigma_{r=1}^{d} K_{i,r} \\
    K_{i,r} &= \begin{cases}
                k_{i,r} & \text{for } 1 = i \\
                a_{1,r} K_{i-1,r} & \text{otherwise}
            \end{cases}
\end{align*}
This ensures that
\begin{equation*}
    \bar{K}_i = \Sigma_{r=1}^{d} k_{1,r} a_{1,r}^i.
\end{equation*}
Now let $\mathbf{B} = [k_{1,1},...,k_{1,d}]^T$, $\mathbf{C} = [1,...,1]^T$, and $\mathbf{A} = \text{diag}(a_{1,1},...,a_{1,d})$. This ensures that
$\bar{K}_i = \mathbf{C}^T \mathbf{A}^i \mathbf{B}$, showing that diagonal SSMs can be recovered by constant-recursive kernels.

\subsection{Wavelet Basis}

In our initial explorations, we parameterized a Haar wavelet basis as a mechanism for producing convolution kernels.
We ultimately did not go with this route, since we found that a simpler solution (direct parameterization) was sufficient.
\fi
\section{Experiment Details\label{sec:exp_details_supp}}

We discuss all the details of our experiments.

\begin{table*}[ht]
    \caption{The values of the best hyperparameters found; LRA, images, language, and time series, and brain fMRI. LR is learning rate and WD is weight decay. BN and LN refer to Batch Normalization and Layer Normalization.
    We use random weight initialization in all runs.}
    \label{tab:hyperparameters}
    \centering
    \tiny
    \begin{tabular}{lccccccccccc}
    \toprule
    {} & {Depth} & {Features $H$} & {Norm} & {kernel LR} & {Dropout} & {$\lambda$} & {Batch Size} & {WD} & {Epochs} & LR \\
    \midrule
    {ListOps} & {8} & {128} & {BN} & {0.0005} & {0.2} & {0.002} & {50} & {0.05} & {40} & 0.01 \\ 
    {Text (IMDB)} & {6} & {256} & {BN} & {0.001} & {0.2} & {0.003} & {16} & {0.05} & {32} & 0.01 \\
    {Retrieval (AAN)} & {6} & {256} & {BN} & {0.0001} & {0.1} & {0.004} & {32} & {0.05} & {20} & 0.01 \\
    {Image} & {6} & {512} & {LN} & {0.001} & {0.2} & {0.003} & {25} & {0.05} & {200} & 0.01 \\
    {Pathfinder} & {6} & {256} & {BN} & {0.001} & {0.3} & {0.001} & {64} & {0.03} & {200} & 0.004 \\
    {Path-X} & {6} & {256} & {BN} & {0.0005} & {0.3} & {0.001} & {4} & {0.05} & {50} & 0.0005 \\
    \midrule
    {sCIFAR} & {6} & {512} & {LN} & {0.001} & {0.2} & {0.001} & {50} & {0.05} & {300} & 0.01 \\
    {2D CIFAR} & {4} & {128} & {LN} & {0.001} & {0} & {0.001} & {50} & {0.01} & {100} & 0.01 \\
    \midrule
    {OpenWebText} & {12} & {768} & {LN} & {0.001} & {0} & {0.001} & {32} & {0.1} & {100B tokens} & 0.0003 \\
    \midrule
    {Time Series} & {3} & {128} & {BN} & {0.001} & {0.2} & {0.003} & {50} & {0.01} & {50} & {1e-5} \\
    \midrule
    \cmidrule{10-11}
    {Brain Upstream} & {4} & {768} & {LN} & {0.001} & {0.2} & {0.0005} & {512} & {0.1} & 5000 steps & 0.01 \\
    {Brain Downstream} & {4} & {768} & {LN} & {0.001} & {0.2} & {0.00005} & {256} & {0.1} & 1000 steps & 0.01 \\
    \bottomrule 
    \end{tabular}
\end{table*}

\paragraph{Hyperparameter Sweeps}
For all methods, we swept the following parameters:
\begin{itemize}
    \item Kernel Dropout: [0.1, 0.2, 0.3, 0.4, 0.5]
    \item Kernel LR: [0.0001, 0.0005, 0.001]
    \item $\lambda$: [0.001, 0.002, 0.003, 0.004, 0.005]
\end{itemize}

\paragraph{Compute Infrastructure}

The experiments in this paper were run on a mixture of different compute platforms.
The LRA experiments, except for Path-X, were swept on a heterogeneous cluster of 1xV100 and 2xV100 nodes.
Path-X and sequential CIFAR were run on single 8xA100 nodes.
The language modeling experiments were run on a single 8xA100 node.
The time series experiments were run on a cluster with 1xP100 nodes.
The brain fMRI experiments were run on a cluster of 2xV100 nodes.

\paragraph{Final Hyperparameters}
Final hyperparameters for reported results are given in Table~\ref{tab:hyperparameters}.

\ifmefomo
\paragraph{Hyperparameter comparision to S4}
\begin{table}[t]
    \caption{Convolution- and SSM-specific hyperparameters.}
    \ifarxiv
    \else
    \scriptsize
    \fi
    \label{tab:hyperparams}
    \centering
    \begin{tabular}{rll} \toprule
    {Model} & {Hyperparameters} & {Initializations} \\ \midrule
    SSM & $d$, $lr_A$, $lr_B$, $lr_C$& LegS, FouT, LegS/FouT \\
    & dropout, discretization & Inv, Lin \\
    Long Convs & $\lambda$, kernel LR, & Random, Geometric \\
    & k, dropout &  \\
    \bottomrule 
    \end{tabular}
\end{table}
Compared to the hyperparameters necessary to train S4, our regularization approaches have significantly fewer hyperparameters and choices than S4.
Convolution-specific hyperparameters for S4 and long convolutions are shown in Table~\ref{tab:hyperparams}.
\else
\fi
\subsection{Functional Magnetic Resonance Imaging Data}
\label{sec:fmri_supp}

Neuroimaging research can be considered as recently entering a big data era, as individual researchers publicly share their collected datasets more frequently.
This development opens up new opportunities for pre-training at scale in neuroimaging research, as recently demonstrated by~\citet{thomas2022self}.
In their work, the authors show that Transformers, pre-trained to predict brain activity for the next time point of input fMRI sequences, outperform other models in learning to identify the mental states (e.g., happiness or fear) underlying new fMRI data.
Recently,~\citet{dao2022hungry} have shown that H3 performs on par with Transformers in this transfer learning paradigm.

To test whether long convolutions also perform on par with SSMs, as implemented in H3, and Transformers in this paradigm, we replicate the analyses of ~\citet{thomas2022self}, using their published fMRI datasets. Conventionally, functional Magnetic Resonance Imaging (fMRI) data are represented in four dimensions, describing the measured blood-oxygen-level-dependent (BOLD) signal as a sequence $S = \{V_1, ..., V_t\}$ of 3-dimensional volumes $V \in \mathbb{R}^{x \times y \times z}$, which show the BOLD signal for each spatial location of the brain (as indicated by the three spatial dimensions $x$, $y$, and $z$).
Yet, due to the strong spatial spatial correlation of brain activity, fMRI data can also be represented differently, by representing individual sequences as a set $\Theta \in {\theta_1, ..., \theta_n}$ of $n$ functionally-independent brain networks $\theta$, where each network describes the BOLD signal for some subset of voxels $v_{x,y,z} \in V$~\citep[e.g.,][]{dadi2020fine}.
The resulting sequences $X \in \mathbb{R}^{t \times n}$ indicate whole-brain activity as a set of $n$ brain networks for $t$ time points~\footnote{\citet{thomas2022self} use $n=1,024$ networks defined by the Dictionaries of Functional Modes~\citep[DiFuMo; ][]{dadi2020fine} Atlas.}.

\begin{figure*}[t]
    \centering
    \includegraphics[width=\textwidth]{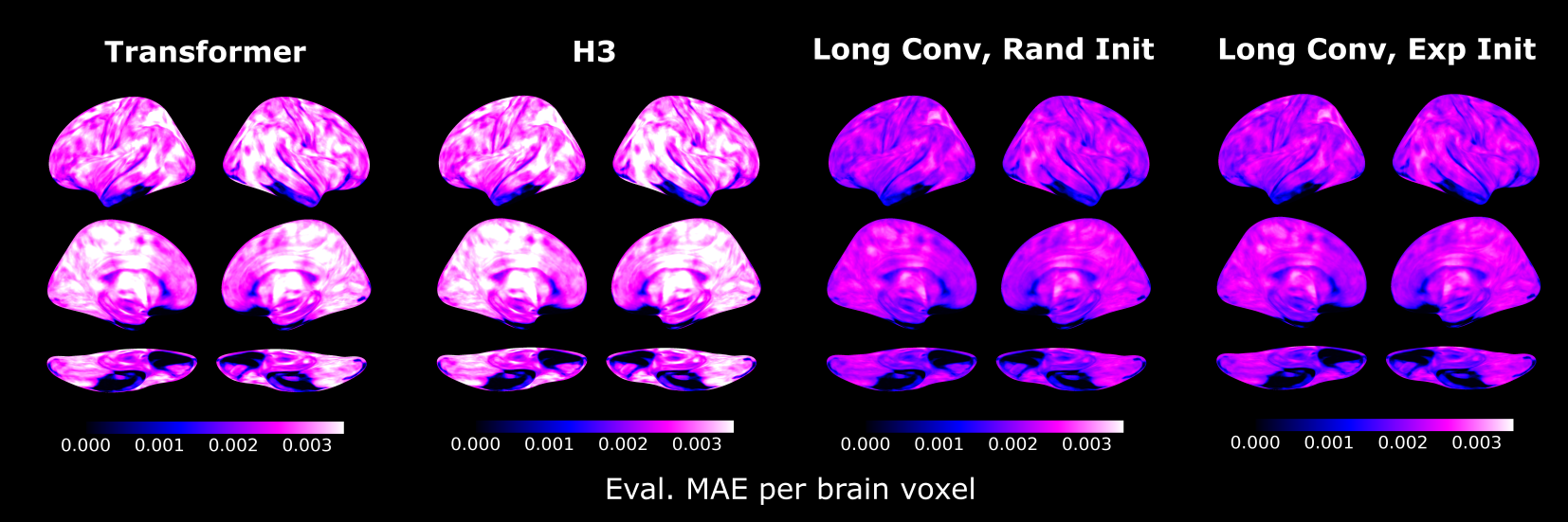}
    \ifarxiv
    \else
    \vspace{-1.5em}
    \fi
    \caption{\label{fig:brainmaps}
    Mean absolute error of pre-trained models in upstream evaluation data for each location of the brain. Brain maps are projected onto the inflated cortical surface of the FsAverage template~\citep{fischl2012freesurfer}.
    }
    \ifarxiv
    \else
    \vspace{-1.5em}
    \fi
\end{figure*}
\paragraph{Upstream learning:} In line with~\citet{thomas2022self}, we pre-train models $f(\cdot)$ to predict whole-brain activity for the next time point $j$ of an fMRI sequence $X$, using a mean absolute error (MAE) training objective, given the model's prediction $\hat{X} \in \mathbb{R}^{t \times n}$: MAE $= \frac{1}{n} \sum_{i=1}^{n} |X_{j,i} - \hat{X}_{j,i}|$; $\hat{X}_{t,n} = b_n + \sum_n f(E^{X})_{t,e} w_{e,n}$; $E^{X}_{t,e} = E^{TR} + E^{pos} + b_e + \sum_n X_{t,n} w_{n,e}$.
Here, $E^{TR} \in \mathbb{R}^{e}$ and $E^{pos} \in \mathbb{R}^{e}$ represent learnable embeddings for each possible time point and position of an input sequence~\citep[for details, see][]{thomas2022self}\footnote{As the sampling frequency of fMRI is variable between datasets, the same position of an input sequence can correspond to different time points.}.
Note that $f(\cdot)$ processes the input in a lower-dimensional representation $E^{X} \in \mathbb{R}^{t \times e}$, where $e=768$, obtained through linear projection.

In line with~\citet{thomas2022self} and~\citet{dao2022flashattention}, we pre-train a Transformer decoder (based on GPT) with $4$ hidden layers and $12$ attention heads and a H3 model with $4$ hidden layers~\citep[with $H=64$ and $m=1$; see ][]{dao2022hungry} in this task.
For both models, the sequence of hidden-states outputs of the last model layer are used to determine $\hat{X}$ (scaled to the original input dimension with linear projection).
We also pre-train variants of H3 that replace its SSM kernel with long convolutions.

We randomly divide the upstream data, which spans fMRI data from $11,980$ experimental runs of $1,726$ individuals, into distinct training and validation datasets by randomly designating $5\%$ of the fMRI runs as validation data and using the rest of the runs for training. 
During training, we randomly sample sequences of $100$ time points from the fMRI runs and train models with the ADAM optimizer (with $\beta_1=0.9$, $\beta_2=0.999$, and $\epsilon=1e^{-8}$ ) for $5,000$ steps at a mini-batch size of $512$ and a learning rate of $5e^{-4}$.
We also apply a linear learning rate decay schedule (with a warm-up phase of $10\%$ of the total number of training steps), gradient norm clipping at $1.0$, $L2$-regularisation (weighted by $0.1$), and dropout at a rate of $0.2$ (throughout all models).
\ifmefomo
The adapted H3 variants clearly outperform the other models in accurately predicting brain activity for the next time point of input sequences (Table~\ref{tab:fmri}).
\else
The adapted H3 variants clearly outperform the other models in accurately predicting brain activity for the next time point of input sequences (Table \ref{tab:fmri} of the main text).
\fi
We also find that the pre-trained models exhibit similar evaluation $MAE$ error distributions throughout the brain, with relatively higher errors in the posterior parietal, occipital, and cingulate cortices as well parts of the limbic system (Fig.~\ref{fig:brainmaps}).

\end{document}